\documentclass{article}
\pdfoutput=1

\usepackage{graphicx} 
\usepackage{subfigure} 

\usepackage{color}
\usepackage{booktabs}
\usepackage[all]{xy}

\usepackage[square,comma,numbers,sort&compress,sectionbib]{natbib}

\usepackage{algorithm}
\usepackage{algorithmic}

\usepackage[normalem]{ulem}

\usepackage{amsmath,amsthm,amssymb}

\usepackage{color}
\definecolor{Blue}{rgb}{0.9,0.3,0.3}
\usepackage{cancel}

\newcommand{\squishlist}{
   \begin{list}{$\bullet$}
    { \setlength{\itemsep}{0pt}      \setlength{\parsep}{3pt}
      \setlength{\topsep}{3pt}       \setlength{\partopsep}{0pt}
      \setlength{\leftmargin}{1.5em} \setlength{\labelwidth}{1em}
      \setlength{\labelsep}{0.5em} } }

\newcommand{\squishlisttwo}{
   \begin{list}{$\bullet$}
    { \setlength{\itemsep}{0pt}    \setlength{\parsep}{0pt}
      \setlength{\topsep}{0pt}     \setlength{\partopsep}{0pt}
      \setlength{\leftmargin}{2em} \setlength{\labelwidth}{1.5em}
      \setlength{\labelsep}{0.5em} } }

\newcommand{\squishend}{
    \end{list}  }

\newcommand{\myvec}[1]{\mathbf{#1}}

\newcommand{\vP}{\myvec{P}}

\newcommand{\vU}{\myvec{U}}

\newcommand{\vW}{\myvec{W}}

\newcommand{\be}{\begin{equation}}
\newcommand{\ee}{\end{equation}}
\newcommand{\bea}{\begin{eqnarray}}
\newcommand{\eea}{\end{eqnarray}}
\newcommand{\beaa}{\begin{eqnarray*}}
\newcommand{\eeaa}{\end{eqnarray*}}

\DeclareMathAlphabet{\mathpzc}{OT1}{pzc}{m}{n}

\usepackage{enumitem}

\newcommand{\note}[1]{}
\renewcommand{\note}[1]{~\\\frame{\begin{minipage}[c]{0.48\textwidth}\vspace{2pt}\center{#1}\vspace{2pt}\end{minipage}}\vspace{3pt}\\}
\newcommand{\hide}[1]{}

\DeclareMathOperator{\argmax}{arg\,max}
\newtheorem{mydefinition}{Definition}

\newtheorem{theorem}[mydefinition]{Theorem}
\newtheorem{remark}[mydefinition]{Remark}
\newtheorem{lemma}[mydefinition]{Lemma}

\renewcommand{\algorithmicrequire}{\textbf{Input:}}
\renewcommand{\algorithmicensure}{\textbf{Return:}}
\renewcommand{\algorithmiccomment}[1]{// \textit{#1}}

\usepackage{hyperref}

\usepackage[accepted]{icml2014}

\icmltitlerunning{Relative Upper Confidence Bound}

\begin{document} 

\twocolumn[
\icmltitle{Relative Upper Confidence Bound for the\\ $K$-Armed Dueling Bandit Problem}

\icmlauthor{Masrour Zoghi}{m.zoghi@uva.nl}
\icmladdress{ISLA, University of Amsterdam, The Netherlands}
\icmlauthor{Shimon Whiteson}{s.a.whiteson@uva.nl}
\icmladdress{ISLA, University of Amsterdam, The Netherlands}
\icmlauthor{Remi Munos}{remi.munos@inria.fr}
\icmladdress{INRIA Lille - Nord Europe, Villeneuve d'Ascq, France}
\icmlauthor{Maarten de Rijke}{derijke@uva.nl}
\icmladdress{ISLA, University of Amsterdam, The Netherlands}

\icmlkeywords{dueling bandits; Thompson sampling}

\vskip 0.3in
]

\begin{abstract} 
This paper proposes a new method for the \emph{$K$-armed dueling bandit problem}, a variation on the regular $K$-armed bandit problem that offers only relative feedback about pairs of arms.  Our approach extends the Upper Confidence Bound algorithm to the relative setting by using estimates of the pairwise probabilities to select a promising arm and applying Upper Confidence Bound with the winner as a benchmark.  
We prove a finite-time regret bound of order $\mathcal O(\log t)$. 
In addition, our empirical results using real data from an information retrieval application show that it greatly outperforms the state of the art.
\end{abstract} 

\section{Introduction}
\label{sec:introduction}


In this paper, we propose and analyze a new algorithm, called Relative Upper Confidence Bound (RUCB), for the \emph{$K$-armed dueling bandit problem} \citep{yue12:k-armed}, a variation on the $K$-armed bandit problem, where the feedback comes in the form of pairwise preferences. We assess the performance of this algorithm using one of the main current applications of the $K$-armed dueling bandit problem, \emph{ranker evaluation} \citep{joachims2002:optimizing,YueJoachims:2011,hofmann:irj13}, which is used in information retrieval, ad placement and recommender systems, among others. 


The $K$-armed dueling bandit problem is part of the general framework of \emph{preference learning} \citep{furnkranz2010,furnkranz2012towards}, where the goal is to learn, not from  real-valued feedback, but from \emph{relative feedback}, which specifies only which of two alternatives is preferred.  Developing effective preference learning methods is important for dealing with domains in which feedback is naturally qualitative (e.g., because it is provided by a human) and specifying real-valued feedback instead would be arbitrary or inefficient \citep{furnkranz2012towards}.



Other algorithms proposed for this problem are Interleaved Filter (IF) \citep{yue12:k-armed}, Beat the Mean (BTM) \citep{YueJoachims:2011}, and SAVAGE \cite{Urvoy:2013}. All of these methods were designed for the \emph{finite-horizon} setting, in which the algorithm requires as input the \emph{exploration horizon}, $T$, the time by which the algorithm needs to produce the best arm. The algorithm is then judged based upon either the \emph{accuracy} of the returned best arm or the \emph{regret} accumulated in the exploration phase.\footnote{These terms are formalized in Section \ref{sec:problemsetting}.} All three of these algorithms use the exploration horizon to set their internal parameters, so for each $T$, there is a separate algorithm $\textup{IF}_T$, $\textup{BTM}_T$ and $\textup{SAVAGE}_T$. By contrast, RUCB does not require this input, making it more useful in practice, since a good exploration horizon is often difficult to guess. Nonetheless, RUCB outperforms these algorithms in terms of the accuracy and regret metrics used in the finite-horizon setting.

The main idea of RUCB is to maintain optimistic estimates of the probabilities of all possible pairwise outcomes, and (1)~use these estimates to select a potential champion, which is an arm that has a chance of being the best arm, and (2)~select an arm to compare to this potential champion by performing regular Upper Confidence Bound \citep{auer2002ucb} relative to it. 


We prove a finite-time high-probability bound of $\mathcal O(\log t)$ on the cumulative regret of RUCB, from which we deduce a bound on the expected cumulative regret. These bounds rely on substantially less restrictive assumptions on the $K$-armed dueling bandit problem than IF and BTM and have better multiplicative constants than those of SAVAGE. Furthermore, our bounds are the first explicitly non-asymptotic results for the $K$-armed dueling bandit problem.  

More importantly, The main distinction of our result is that it holds for \emph{all} time steps.  By contrast, given an exploration horizon $T$, the results for IF, BTM and SAVAGE bound only the regret accumulated by $\textup{IF}_T$, $\textup{BTM}_T$ and $\textup{SAVAGE}_T$ in the first $T$ time steps. 

Finally, we evaluate our method empirically using real data from an information retrieval application.  The results show that RUCB can learn quickly and effectively and greatly outperforms BTM and SAVAGE. 

The main contributions of this paper are as follows:

\vspace{-2mm}

\begin{itemize}[leftmargin=*]
\item A novel algorithm for the $K$-armed dueling bandit problem that is more broadly applicable than existing algorithms,

\vspace{-1mm}

\item More comprehensive theoretical results that make less restrictive assumptions than those of IF and BTM, have better multiplicative constants than the results of SAVAGE, and apply to all time steps, and

\vspace{-1mm}

\item Experimental results, based on a real-world application, demonstrating the superior performance of our algorithm compared to existing methods.
\end{itemize}

\section{Problem Setting}
\label{sec:problemsetting}



The \emph{$K$-armed dueling bandit} problem \cite{yue12:k-armed} is a modification of the \emph{$K$-armed bandit} problem \cite{auer2002ucb}: the latter considers $K$ arms $\{a_1,\ldots,a_K\}$ and at each \emph{time-step}, an arm $a_i$ can be \emph{pulled}, generating a \emph{reward} drawn from an unknown stationary distribution with expected value $\mu_i$. 
The $K$-armed \emph{dueling} bandit problem is a variation, where instead of pulling a single arm, we choose a pair $(a_i,a_j)$ and receive one of the two as the better choice, with the probability of $a_i$ being picked equal to a constant $p_{ij}$ and that of $a_j$ equal to $p_{ji} = 1 - p_{ij}$. We define the \emph{preference matrix} $\vP=\left[p_{ij}\right]$, whose $ij$ entry is $p_{ij}$. 

In this paper, we assume that there exists a \emph{Condorcet winner} \cite{Urvoy:2013}: an arm, which without loss of generality we label $a_1$, such that $p_{1i} > \frac{1}{2}$ for all $i>1$.  
Given a Condorcet winner, we define \emph{regret} for each time-step as follows \cite{yue12:k-armed}: if arms $a_i$ and $a_j$ were chosen for comparison at time $t$, then regret at that time is set to be $r_t := \frac{\Delta_{1i}+\Delta_{1j}}{2}$, with $\Delta_{k} := p_{1k}-\frac{1}{2}$ for all $k \in \{1,\ldots,K\}$. Thus, regret measures the average advantage that the Condorcet winner has over the two arms being compared against each other. Given our assumption on the probabilities $p_{1k}$, this implies that $r=0$ if and only if the best arm is compared against itself. We define \emph{cumulative regret up to time} $T$ to be $R_T = \sum_{t=1}^T r_t$.

The Condorcet winner is different in a subtle but important way from the \emph{Borda winner} \cite{Urvoy:2013}, which is an arm $a_b$ that satisfies $\sum_j p_{bj} \geq \sum_j p_{ij}$, for all $i=1,\ldots,K$.
In other words, when averaged across all other arms, the Borda winner is the arm with the highest probability of winning a given comparison. In the $K$-armed dueling bandit problem, the Condorcet winner is sought rather than the Borda winner, for two reasons. First, in many applications, including the ranker evaluation problem addressed in our experiments, the eventual goal is to adapt to the preferences of the users of the system. Given a choice between the Borda and Condorcet winners, those users prefer the latter in a direct comparison, so it is immaterial how these two arms fare against the others. Second, in settings where the Borda winner is more appropriate, no special methods are required: one can simply solve the $K$-armed bandit algorithm with arms $\{a_1,\ldots,a_K\}$, where pulling $a_i$ means choosing an index $j \in \{1,\ldots,K\}$ randomly and comparing $a_i$ against $a_j$. Thus, research on the $K$-armed dueling bandit problem focuses on finding the Condorcet winner, for which special methods are required to avoid mistakenly choosing the Borda winner.

The goal of a bandit algorithm can be formalized in several ways. In this paper, we consider two standard settings:

\vspace{-3mm}

\begin{enumerate}[leftmargin=*]
\item \emph{The finite-horizon setting}: In this setting, the algorithm is told in advance the exploration \emph{horizon}, $T$, i.e., the number of time-steps that the evaluation process is given to explore before it has to produce a single arm as the best, which will be exploited thenceforth. In this setting, the algorithm can be assessed on its \emph{accuracy}, the probability that a given run of the algorithm reports the Condorcet winner as the best arm \cite{Urvoy:2013}, which is related to expected \emph{simple regret}: the regret associated with the algorithm's choice of the best arm, i.e., $r_{T+1}$ \cite{Bubeck:2009}. Another measure of success in this setting is the amount of regret accumulated during the exploration phase, as formulated by the \emph{explore-then-exploit} problem formulation \cite{yue12:k-armed}.


\item \emph{The horizonless setting}: In this setting, no horizon is specified and the evaluation process continues indefinitely.  Thus, it is no longer sufficient for the algorithm to maximize accuracy or minimize regret after a single horizon is reached. Instead, it must minimize regret across \emph{all} horizons by rapidly decreasing the frequency of comparisons involving suboptimal arms, particularly those that fare worse in comparison to the best arm. This goal can be formulated as minimizing the cumulative regret over time, rather than with respect to a fixed horizon \cite{lai85:bandit-lb}. 
\end{enumerate}

\vspace{-2mm}

As we describe in Section \ref{sec:relatedwork}, all existing $K$-armed dueling bandit methods target the finite-horizon setting. However, we argue that the horizonless setting is more relevant in practice for the following reason: finite-horizon methods require a horizon as input and often behave differently for different horizons.  This poses a practical problem because it is typically difficult to know in advance how many comparisons are required to determine the best arm with confidence and thus how to set the horizon.  If the horizon is set too long, the algorithm is too exploratory, increasing the number of evaluations needed to find the best arm.  If it is set too short, the best arm remains unknown when the horizon is reached and the algorithm must be restarted with a longer horizon. 

Moreover, any algorithm that can deal with the horizonless setting can easily be modified to address the finite-horizon setting by simply stopping the algorithm when it reaches the horizon and returning the best arm. By contrast, for the reverse direction, one would have to resort to the ``doubling trick'' \citep[Section 2.3]{Cesa-Bianchi:2006}, which leads to substantially worse regret results: this is because all of the upper bounds proven for methods addressing the finite-horizon setting so far are in $\mathcal O(\log T)$ and applying the doubling trick to such results would lead to regret bounds of order $(\log T)^2$, with the extra log factor coming from the number of partitions.

To the best of our knowledge, RUCB is the first $K$-armed dueling bandit algorithm that can function in the horizonless setting without resorting to the doubling trick. We show in Section \ref{sec:algorithm} how it can be adapted to the finite-horizon setting.



\section{Related Work}
\label{sec:relatedwork}


In this section, we briefly survey existing methods for the $K$-armed dueling bandit problem.


The first method for the $K$-armed dueling bandit problem is \emph{interleaved filter} (IF) \cite{yue12:k-armed}, which was designed for a finite-horizon scenario and which proceeds by picking a \emph{reference} arm to compare against the rest and using it to eliminate other arms, until the reference arm is eliminated by a better arm, in which case the latter becomes the reference arm and the algorithm continues as before. The algorithm terminates either when all other arms are eliminated or if the exploration horizon $T$ is reached.


More recently, the \emph{beat the mean} (BTM) algorithm has been shown to outperform IF \cite{YueJoachims:2011}, while imposing less restrictive assumptions on the $K$-armed dueling bandit problem. BTM focuses exploration on the arms that have been involved in the fewest comparisons.  When it determines that an arm fares on average too poorly in comparison to the remaining arms, it removes it from consideration. More precisely, BTM considers the performance of each arm against the \emph{mean arm} by averaging the arm's scores against all other arms and uses these estimates to decide which arm should be eliminated.

Both IF and BTM require the comparison probabilities $p_{ij}$ to satisfy certain conditions that are difficult to verify without specific knowledge about the dueling bandit problem at hand and, moreover, are  often violated in practice 
(see the supplementary material for a more thorough discussion and analysis of these assumptions). Under these conditions, theoretical results have been proven for IF and BTM in \cite{yue12:k-armed} and \cite{YueJoachims:2011}. More precisely, both algorithms take the exploration horizon $T$ as an input and so for each $T$, there are algorithms $\textup{IF}_T$ and $\textup{BTM}_T$; the results then state the following: for large $T$, in the case of $\textup{IF}_T$, we have the expected regret bound

\vspace{-4mm}

\[ \mathbb{E}\left[R^{\textup{IF}_T}_T\right] \leq C \frac{K \log T}{\min_{j=2}^K \Delta_j}, \]
and, in the case of $\textup{BTM}_T$, the high probability regret bound

\vspace{-4mm}

\[ R^{\textup{BTM}_T}_T \leq C^{'} \frac{\gamma^7 K \log T}{\min_{j=2}^K \Delta_j} \textup{ with high probability,} \]

where arm $a_1$ is assumed to be the best arm, and we define $\Delta_{j} := p_{1j}-\frac{1}{2}$, and $C$ and $C^{'}$ are constants independent of the specific dueling bandit problem. 

The first bound matches a lower bound proven in \citep[Theorem 4]{yue12:k-armed}. However, as pointed out in \cite{YueJoachims:2011}, this result holds for a very restrictive class of $K$-armed dueling bandit problems. In an attempt to remedy this issue, the second bound was proven for BTM, which includes a relaxation parameter $\gamma$ that allows for a broader class of problems, as discussed in the supplementary material. The difficulty with this result is that the parameter $\gamma$, which depends on the probabilities $p_{ij}$ and must be passed to the algorithm, can be very large. Since it is raised to the power of $7$, this makes the bound very loose. For instance, in the three-ranker evaluation experiments discussed in Section \ref{sec:experiments}, the values for $\gamma$ are $4.85$, $11.6$ and $47.3$ for the $16$-, $32$- and $64$-armed examples.

In contrast to the above limitations and loosenesses, in Section \ref{sec:theory} we provide \emph{explicit} bounds on the regret accumulated by RUCB that do not depend on $\gamma$ and require only the existence of a Condorcet winner for their validity, which makes them much more broadly applicable.

Sensitivity Analysis of VAriables for Generic Exploration (SAVAGE) \cite{Urvoy:2013} is a recently proposed algorithm that outperforms both IF and BTM by a wide margin when the number of arms is of moderate size. Moreover, one version of SAVAGE, called \emph{Condorcet SAVAGE}, makes the Condorcet assumption and performed the best experimentally \cite{Urvoy:2013}. Condorcet SAVAGE compares pairs of arms uniformly randomly until there exists a pair for which one of the arms beats another by a wide margin, in which case the loser is removed from the pool of arms under consideration. We show in this paper that our proposed algorithm for ranker evaluation substantially outperforms Condorcet SAVAGE.

The theoretical result proven for Condorcet SAVAGE has the following form \citep[Theorem 3]{Urvoy:2013}. First, let us assume that $a_1$ is the Condorcet winner and let $\widehat{T}_{\textup{CSAVAGE}_T}$ denote the number of iterations the Condorcet SAVAGE algorithm with exploration horizon $T$ requires before terminating and returning the best arm; then, given $\delta > 0$, with probability $1-\delta$, we have for large $T$

\vspace{-5mm}

\begin{equation*}
\widehat{T}_{\textup{CSAVAGE}_T} \leq C^{''} \sum_{j=1}^{K-1} \frac{j \cdot \log\left(\frac{KT}{\delta}\right)}{\Delta_{j+1}^2},
\end{equation*}

\vspace{-2mm}

with the indices $j$ arranged such that $\Delta_2 \leq \cdots \leq \Delta_K$ and $\Delta_j = p_{1j}-\frac{1}{2}$ as before, and $C^{''}$ a problem independent constant. This bound is very similar in spirit to our high probability result, with the important distinction that, unlike the above bound, the multiplicative factors in our result (i.e., the $D_{ij}$ in Theorem \ref{thm:HighProbBound} below) do not depend on $\delta$. Moreover, in \citep[Appendix B.1]{Urvoy:2013}, the authors show that for large $T$ we have the following expected regret bound:

\vspace{-3mm}

\[ \mathbb{E}\left[R^{\textup{CSAVAGE}_T}_T\right] \leq C^{''} \sum_{j=2}^K \frac{j \cdot \log\left(KT^2\right)}{\Delta_{j}^2} + 1. \]

\vspace{-2mm}

This is similar to our expected regret bound in Theorem \ref{thm:ExpBound}, although for difficult problems where the $\Delta_j$ are small, Theorem \ref{thm:ExpBound} yields a tighter bound due to the presence of the $\Delta_j$ in the numerator of the second summand.

An important advantage that our result has over the results reviewed here is an explicit expression for the additive constant, which was left out of the analyses of IF, BTM and SAVAGE. 

Finally, note that all of the above results bound only $R_T$, where $T$ is the predetermined exploration horizon, since IF, BTM and SAVAGE were designed for the finite-horizon setting. By contrast, in Section \ref{sec:theory}, we bound the cumulative regret of each version of our algorithm for \emph{all} time steps.

\section{Method}
\label{sec:algorithm}


We now introduce Relative Upper Confidence Bound (RUCB), which is applicable to any $K$-armed dueling bandit problem with a Condorcet winner.

\vspace{-2mm}

\begin{algorithm}[h]
\caption{Relative Upper Confidence Bound}
\label{alg:RUCB}
\begin{algorithmic}[1]
{
\REQUIRE $\alpha > \frac{1}{2}$, $T \in \{1,2,\ldots\} \cup \{\infty\}$
\STATE $\vW = \left[w_{ij}\right] \gets \mathbf{0}_{K \times K} \; $ // 2D array of wins: $w_{ij}$ is the number of times $a_i$ beat $a_j$
\FOR{$t=1,\dots,T$}
   \STATE $\vU := \left[u_{ij}\right] = \frac{\vW}{\vW+\vW^T} + \sqrt{\frac{\alpha\ln t}{\vW+\vW^T}}$ \; // All operations are element-wise; $\frac{x}{0}:=1$ for any $x$. 
   \STATE $u_{ii} \gets \frac{1}{2}$ for each $i=1,\ldots,K$.
   \STATE Pick any $c$ satisfying $u_{cj} \geq \frac{1}{2}$ for all $j$. If no such $c$, pick $c$ randomly from $\{1,\ldots,K\}$.
   \STATE $d \gets \displaystyle\argmax_j u_{jc}$ 
   \STATE Compare arms $a_c$ and $a_d$ and increment $w_{cd}$ or $w_{dc}$ depending on which arm wins.
\ENDFOR
\ENSURE An arm $a_c$ that beats the most arms, i.e., $c$ with the largest count~$\# \left\{ j | \frac{w_{cj}}{w_{cj}+w_{jc}} > \frac{1}{2} \right\}$. 
}
\end{algorithmic}
\end{algorithm}

\vspace{-1mm}

In each time-step, RUCB, shown in Algorithm \ref{alg:RUCB}, goes through the following three stages: 

%
%
%
%
%

(1) RUCB puts all arms in a pool of potential champions. Then, it compares each arm $a_i$ against all other arms optimistically: for all $i \neq j$, we compute the upper bound $u_{ij}(t) = \mu_{ij}(t) + c_{ij}(t)$, where  $\mu_{ij}(t)$ is the frequentist estimate of $p_{ij}$ at time $t$ and $c_{ij}(t)$ is an optimism bonus that increases with $t$ and decreases with the number of comparisons between $i$ and $j$ (Line 3). If we have $u_{ij} < \frac{1}{2}$ for any $j$, then $a_i$ is removed from the pool. Next, a champion arm $a_c$ is chosen randomly from the remaining potential champions (Line 5).


(2) Regular UCB is performed using $a_c$ as a benchmark (Line 6), i.e., UCB is performed on the set of arms $a_{1c} \ldots a_{Kc}$.  Specifically, we select the arm $d = \argmax_j u_{jc}$.  When $c \neq j$, $u_{jc}$ is defined as above.  When $c = j$, since $p_{cc}=\frac{1}{2}$, we set $u_{cc}=\frac{1}{2}$ (Line 4).




(3) The pair $(a_c,a_d)$ are compared and the score sheet is updated as appropriate (Line 7).

Note that in stage (1) the comparisons are based on $u_{cj}$, i.e., $a_c$ is compared optimistically to the other arms, making it easier for it to become the champion. By contrast, in stage (2) the comparisons are based on $u_{jc}$, i.e., $a_c$ is compared to the other arms pessimistically, making it more difficult for $a_c$ to be compared against itself. This is important because comparing an arm against itself yields no information. Thus, RUCB strives to avoid auto-comparisons until there is great certainty that $a_c$ is indeed the Condorcet winner.


Eventually, as more comparisons are conducted, the estimates $\mu_{1j}$ tend to concentrate above $\frac{1}{2}$ and the optimism bonuses $c_{1j}(t)$ will become small.  Thus, both stages of the algorithm will increasingly select $a_1$, i.e., $a_c = a_d = a_1$. Since comparing $a_1$ to itself is optimal, $r_t$ declines over time. 

Note that Algorithm \ref{alg:RUCB} is a finite-horizon algorithm if $T < \infty$ and a horizonless one if $T = \infty$, in which case the for loop never terminates.

\section{Theoretical Results}
\label{sec:theory}


In this section, we prove finite-time high-probability and expected regret bounds for RUCB.  We first state Lemma \ref{lem:HighProbBound} and use it to prove a high-probability bound in Theorem \ref{thm:HighProbBound}, from which we deduce an expected regret bound in Theorem \ref{thm:ExpBound}.


To simplify notation, we assume without loss of generality that $a_1$ is the optimal arm in the following. Moreover, given any $K$-armed dueling bandit algorithm, we define $w_{ij}(t)$ to be the number of times arm $a_i$ has beaten $a_j$ in the first $t$ iterations of the algorithm. We also define $u_{ij}(t) := \frac{w_{ij}(t)}{w_{ij}(t)+w_{ji}(t)} + \sqrt{\frac{\alpha\ln t}{w_{ij}(t)+w_{ji}(t)}}$, for any given $\alpha > 0$, and set $l_{ij}(t) := 1-u_{ji}(t)$. Moreover, for any $\delta > 0$, define $C(\delta) := \left(\frac{(4\alpha-1)K^2}{(2\alpha-1)\delta}\right)^{\frac{1}{2\alpha-1}}$.

\begin{lemma}\label{lem:HighProbBound} Let $\vP := \left[ p_{ij} \right]$ be the preference matrix of a $K$-armed dueling bandit problem with arms $\{a_1,\ldots,a_K\}$, satisfying $p_{1j} > \frac{1}{2}$ for all $j > 1$ (i.e. $a_1$ is the Condorcet winner). Then, for any dueling bandit algorithm and any $\alpha > \frac{1}{2}$ and $\delta > 0$, we have

\vspace{-7mm}

\begin{equation*}  P\Big( \forall\,t>C(\delta),i,j,\; p_{ij} \in [l_{ij}(t),u_{ij}(t)] \Big) > 1-\delta. \end{equation*}
\end{lemma}

\vspace{-5mm}

\begin{proof}
See the supplementary material.
\end{proof}

\vspace{-3mm}

The idea behind this lemma is depicted in Figure \ref{fig:lemma1}, which illustrates the two phenomena that make it possible: first, as long as arms $a_i$ and $a_j$ are not compared against each other, the interval $[l_{ij}(t),u_{ij}(t)]$ will grow in length as $\sqrt{\log t}$, hence approaching $p_{ij}$; second, as the number of comparisons between $a_i$ and $a_j$ increases, the estimated means $\mu_{ij}$ approach $p_{ij}$, hence increasing the probability that the interval $[l_{ij}(t),u_{ij}(t)]$ will contain $p_{ij}$. 

\begin{figure}[!t]
\centering
\includegraphics[width=.485\textwidth]{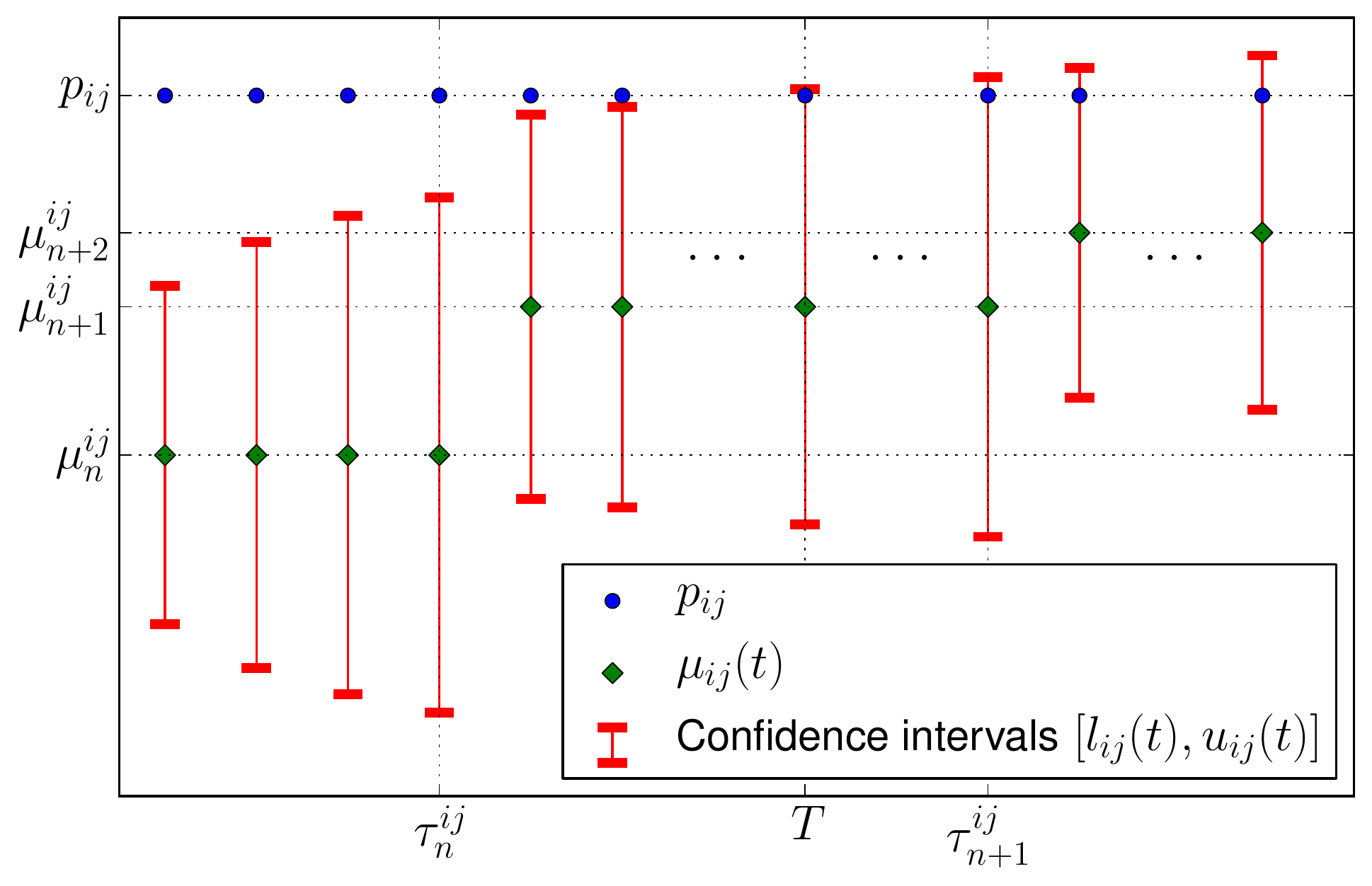}

\vspace{-5mm}

\caption{An illustration of the idea behind Lemma \ref{lem:HighProbBound} using an example of how the confidence intervals of a single pair of arms $(a_i,a_j)$, and their relation to the comparison probability $p_{ij}$, might evolve over time. The time-step $\tau^{ij}_m$ denotes the $m^{\textup{th}}$ time when the arms $a_i$ and $a_j$ were chosen by RUCB to be compared against each other. We also define $\mu^{ij}_m := \mu_{ij}(\tau^{ij}_m)$. The time $T$ is when the confidence intervals begin to include $p_{ij}$. The lemma then states that with probability $1-\delta$, we have $T \leq C(\delta)$.}


\label{fig:lemma1}
\end{figure}

Let us now turn to our high probability bound:

\begin{theorem}\label{thm:HighProbBound} Given a preference matrix $\vP = [p_{ij}]$ and $\delta > 0$ and $\alpha > \frac{1}{2}$, define $C(\delta) := \left(\frac{(4\alpha-1)K^2}{(2\alpha-1)\delta}\right)^{\frac{1}{2\alpha-1}}$ and $D_{ij} := \frac{4\alpha}{\min\{\Delta_i^2,\Delta_j^2\}}$ for each $i,j = 1,\ldots,K$ with $i \neq j$, where $\Delta_i := \frac{1}{2}-p_{i1}$, and set $D_{ii} = 0$ for all $i$. Then, if we apply Algorithm \ref{alg:RUCB} to the $K$-armed dueling bandit problem defined by $\vP$, given any pair $(i,j) \neq (1,1)$, the number of comparisons between arms $a_i$ and $a_j$ performed up to time $t$, denoted by $N_{ij}(t)$, satisfies

\vspace{-7mm}

\begin{equation}
\label{bnd:HighProbCount}
P\bigg(\forall\,t,\; N_{ij}(t) \leq \max\Big\{C(\delta),D_{ij}\ln t\Big\} \bigg) > 1-\delta.
\end{equation}

\vspace{-3mm}

Moreover, we have the following high probability bound for the regret accrued by the algorithm:

\vspace{-7mm}

\begin{equation}\label{bnd:HighProbReg} P\bigg(\forall\,t,\; R_t \leq C(\delta)\Delta^*+\sum_{i>j} D_{ij} \Delta_{ij} \ln t \bigg) > 1-\delta, \end{equation}  

\vspace{-4mm}

where $\Delta^* := \max_i \Delta_i$ and $\Delta_{ij} := \frac{\Delta_i+\Delta_j}{2}$, while $R_t$ is the cumulative regret as defined in Section \ref{sec:problemsetting}.
\end{theorem}

\begin{proof}

\vspace{-2mm}

Given Lemma \ref{lem:HighProbBound}, we know with probability $1-\delta$ that $p_{ij} \in [l_{ij}(t),u_{ij}(t)]$ for all $t > C(\delta)$. Let us first deal with the easy case when $i = j \neq 1$: when $t > C(\delta)$ holds, $a_i$ cannot be played against itself, since if we get $c=i$ in Algorithm \ref{alg:RUCB}, then by Lemma \ref{lem:HighProbBound} and the fact that $a_1$ is the Condorcet winner we have 
\[ u_{ii}(t) = \frac{1}{2} < p_{1i} \leq u_{1i}(t), \] 
and so $d \neq i$.

Now, let us assume that distinct arms $a_i$ and $a_j$ have been compared against each other more than $D_{ij}\ln t$ times and that $t > C(\delta)$. If $s$ is the last time $a_i$ and $a_j$ were compared against each other, we must have 

\vspace{-6mm}

\begin{align}\label{ineq:LUCB1}
& u_{ij}(s)-l_{ij}(s) = 2\sqrt{\frac{\alpha\ln s}{N_{ij}(t)}} \\
& \qquad \leq 2\sqrt{\frac{\alpha\ln t}{N_{ij}(t)}} < 2\sqrt{\frac{\alpha\ln t}{\frac{4\alpha\ln t}{\min\{\Delta_i^2,\Delta_j^2\}}}} = \min\{\Delta_i,\Delta_j\}. \nonumber
\end{align}

\vspace{-3mm}

On the other hand, for $a_i$ to have been compared against $a_j$ at time $s$, one of the following two scenarios must have happened:


\begin{figure}[!tb]
\centering
\includegraphics[width=.385\textwidth]{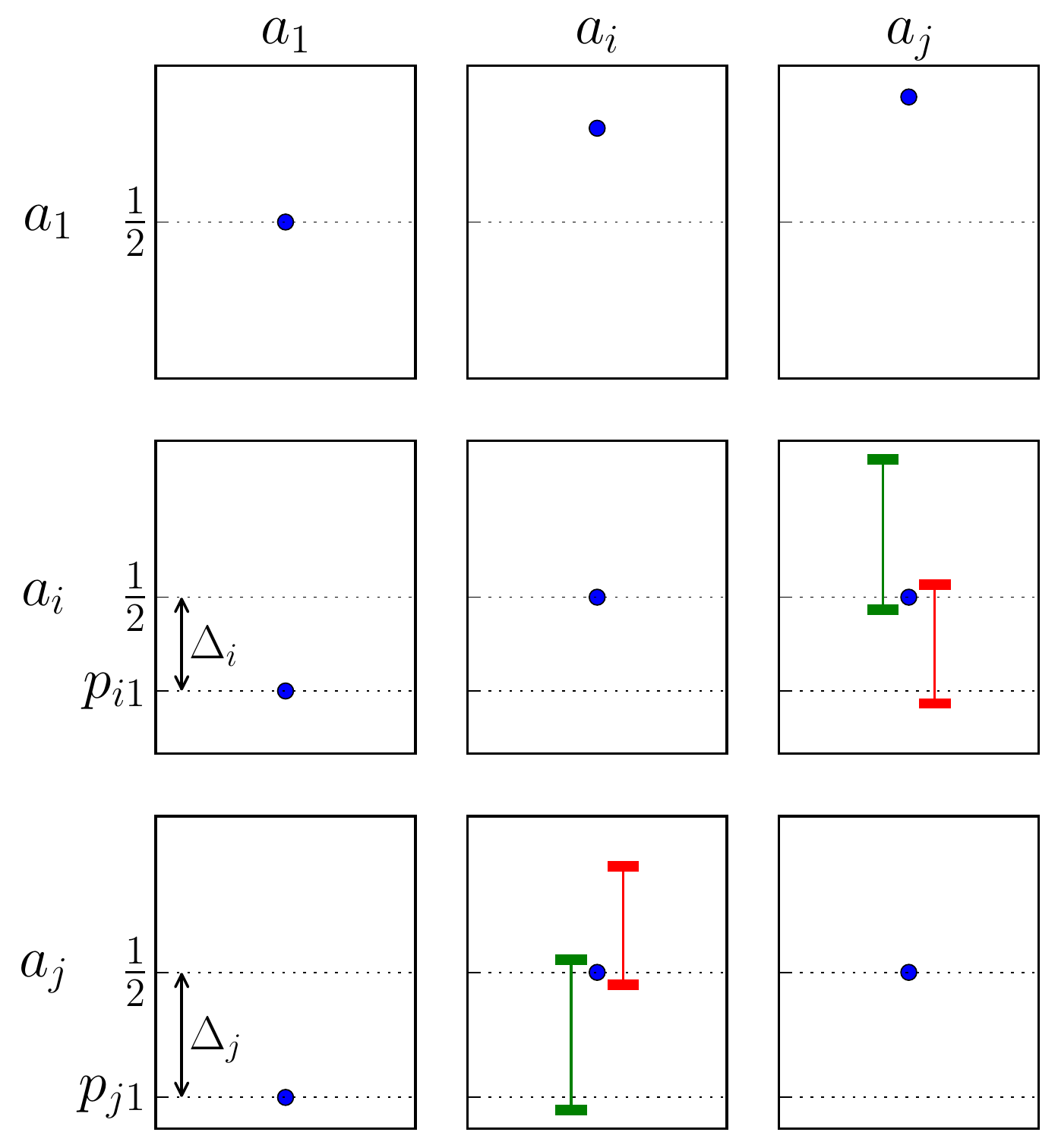}

\vspace{-3mm}

\caption{An illustration of the proof of Theorem \ref{thm:HighProbBound}. The figure shows an example of the internal state of RUCB at time $s$. The height of the dot in the block in row $a_m$ and column $a_n$ represents the comparisons probability $p_{mn}$, while the interval, where present, represents the confidence interval $[l_{mn},u_{mn}]$: we have only included them in the $(a_i,a_j)$ and the $(a_j,a_i)$ blocks of the figure because those are the ones that are discussed in the proof. Moreover, in those blocks, we have included the outcomes of two different runs: one drawn to the left of the dots representing $p_{ij}$ and $p_{ji}$, and the other to the right (the horizontal axis in these plots has no other significance). These two outcomes are included to address the dichotomy present in the proof. Note that for a given run, we must have $[l_{ji}(s),u_{ji}(s)] = [1-u_{ij}(s),1-l_{ij}(s)]$ for any time $s$, hence the symmetry present in this figure.} 


\label{fig:theorem2}
\end{figure}

\begin{itemize}[leftmargin=*,topsep=0pt,parsep=0pt,partopsep=0pt]
	\item[I.] In Algorithm \ref{alg:RUCB}, we had $c=i$ and $d=j$, in which case both of the following inequalities must hold:
	\begin{itemize}[leftmargin=*] 
		\item[a.] $u_{ij}(s) \geq \frac{1}{2}$, since otherwise $c$ could not have been set to $i$ by Line 5 of Algorithm \ref{alg:RUCB}, and
		\item[b.] $l_{ij}(s) = 1-u_{ji}(s) \leq 1-p_{1i} = p_{i1}$, since we know that $p_{1j} \leq u_{1i}(t)$, by Lemma \ref{lem:HighProbBound} and the fact that $t>C(\delta)$, and for $d=j$ to be satisfied, we must have $u_{1i}(t) \leq u_{ji}(t)$ by Line 6 of Algorithm \ref{alg:RUCB}.
	\end{itemize}
	From these two inequalities, we can conclude

	\vspace{-2mm}

	\begin{equation}\label{ineq:LUCB2a} u_{ij}(s) - l_{ij}(s) \geq \frac{1}{2}-p_{i1} = \Delta_i. \end{equation}


	This inequality is illustrated using the lower right confidence interval in the $(a_i,a_j)$ block of Figure \ref{fig:theorem2}, where the interval shows $[l_{ij}(s),u_{ij}(s)]$ and the distance between the dotted lines is $\frac{1}{2}-p_{i1}$.

	\item[II.] In Algorithm \ref{alg:RUCB}, we had $c=j$ and $d=i$, in which case swapping $i$ and $j$ in the above argument gives

	\vspace{-2mm}

	\begin{equation}\label{ineq:LUCB2b} u_{ji}(s) - l_{ji}(s) \geq \frac{1}{2}-p_{j1} = \Delta_j. \end{equation}


	Similarly, this is illustrated using the lower left confidence interval in the $(a_j,a_i)$ block of Figure \ref{fig:theorem2}, where the interval shows $[l_{ji}(s),u_{ji}(s)]$ and the distance between the dotted lines is $\frac{1}{2}-p_{j1}$.
\end{itemize}
Putting \eqref{ineq:LUCB2a} and \eqref{ineq:LUCB2b} together with \eqref{ineq:LUCB1} yields a contradiction, so with probability $1-\delta$ we cannot have $N_{ij}$ be larger than both $C(\delta)$ and $D_{ij}\ln t$.

This gives us \eqref{bnd:HighProbCount}, from which \eqref{bnd:HighProbReg} follows by allowing for the largest regret, $\Delta^*$, to occur in each of the first $C(\delta)$ steps of the algorithm and adding the regret accrued by $D_{ij}\ln t$ comparisons between $a_i$ and $a_j$. \qedhere

\end{proof}

\begin{figure}[!b]


\centering
\includegraphics[width=.485\textwidth]{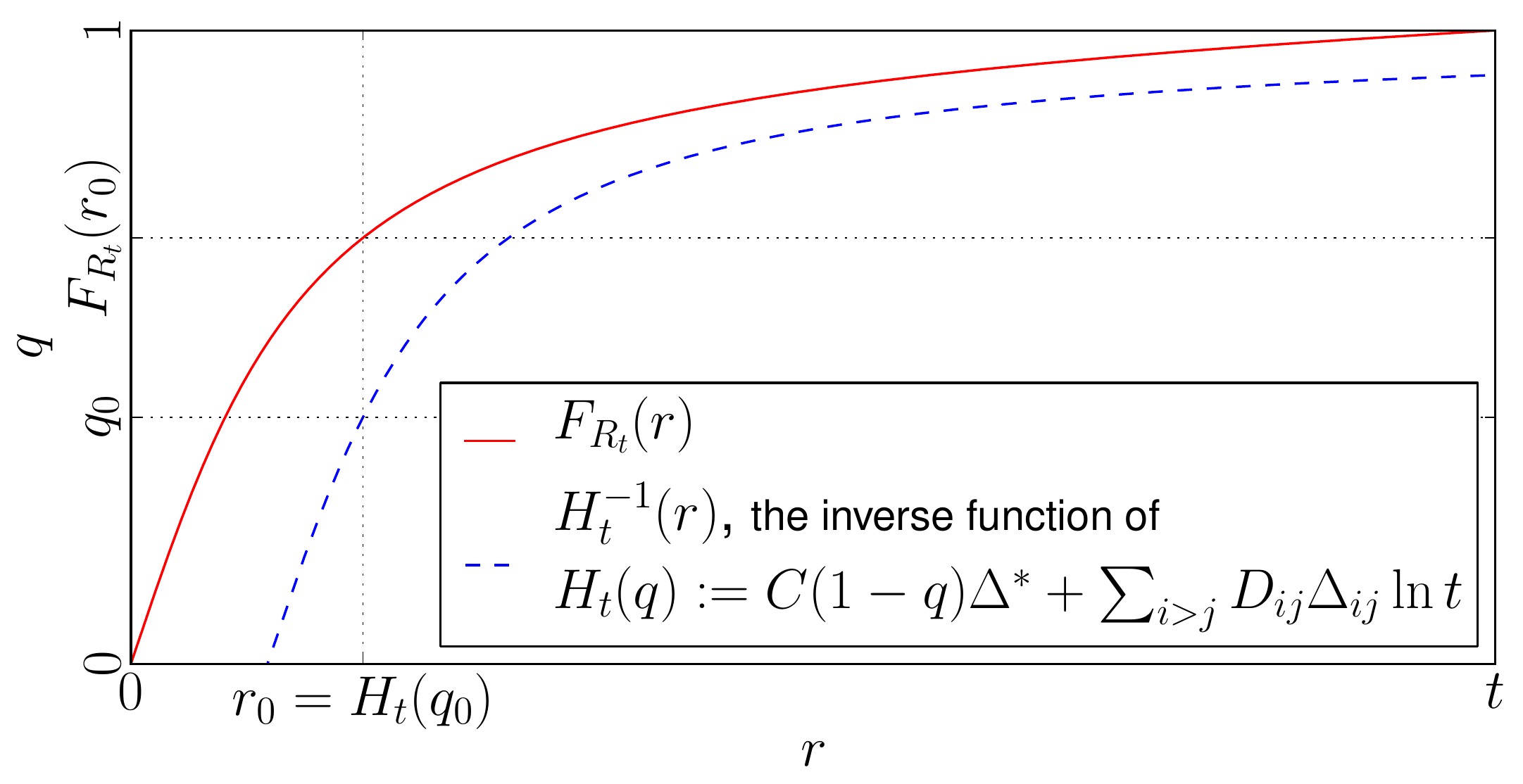}

\vspace{-3mm}

\caption{A schematic graph illustrating the proof of Theorem \ref{thm:ExpBound}. Note that the expression for $H_t(q)$ is extracted from \eqref{bnd:HighProbReg}, which also implies that $H_t^{-1}$ is necessarily below $F_{R_t}$: formulated in terms of CDFs, \eqref{bnd:HighProbReg} states that $F_{R_t}\left(H_t(q_0)\right) > q_0 = H_t^{-1}\left(H_t(q_0)\right)$, where $q_0=1-\delta_0$ is a quantile. From this, we can conclude that $F_{R_t}(r) > H_t^{-1}(r)$ for all $r$.}

\vspace{-3mm}

\label{fig:theorem3}
\end{figure}

Next, we prove our expected regret bound:

\begin{theorem}\label{thm:ExpBound}
Given $\alpha > 1$, the expected regret accumulated by RUCB after $t$ iterations is bounded by

\vspace{-7mm}

\begin{align}\label{bnd:ExpReg} 
\mathbb{E}[R_t] & \leq \Delta^*\left(\frac{(4\alpha-1)K^2}{2\alpha-1}\right)^{\frac{1}{2\alpha-1}} \frac{2\alpha-1}{2\alpha-2} \nonumber \\
 & \qquad +\sum_{i>j}2\alpha\frac{\Delta_i+\Delta_j}{\min\{\Delta_i^2,\Delta_j^2\}}\ln t. 
\end{align}

\vspace{-5mm}

\end{theorem}

\begin{proof}
We can obtain the bound in \eqref{bnd:ExpReg} from \eqref{bnd:HighProbReg} by integrating with respect to $\delta$ from $0$ to $1$. This is because given any one-dimensional random variable $X$ with CDF $F_X$, we can use the identity $\mathbb{E}[X] = \int_0^1 F_X^{-1}(q)dq$. In our case, $X=R_t$ for a fixed time $t$ and, as illustrated in Figure \ref{fig:theorem3}, we can deduce from \eqref{bnd:HighProbReg} that $F_{R_t}(r) > H_t^{-1}(r)$, which gives the bound
\[ F_{R_t}^{-1}(q) < H_t(q) = C(1-q)\Delta^*+\sum_{i>j} D_{ij} \Delta_{ij} \ln t. \]

\vspace{-3mm}

Now, assume that $\alpha > 1$. To derive \eqref{bnd:ExpReg} from the above inequality, we need to integrate the righthand side, and since it is only the first term in the summand that depends on $q$, that is all we need to integrate. To do so, recall that $C(\delta) := \left(\frac{(4\alpha-1)K^2}{(2\alpha-1)\delta}\right)^{\frac{1}{2\alpha-1}}$, so to simplify notation, we define $L := \left(\frac{(4\alpha-1)K^2}{2\alpha-1}\right)^{\frac{1}{2\alpha-1}}$. Now, we can carry out the integration as follows, beginning by using the substitution $1-q=\delta$, $dq = -d\delta$:
\begin{align*}
& \int_{q=0}^1 C(1-q) dq = \int_{\delta=1}^0 -C(\delta) d\delta \\
& = \int_0^1 \left(\frac{(4\alpha-1)K^2}{(2\alpha-1)\delta}\right)^{\frac{1}{2\alpha-1}} d\delta = L \int_0^1 \delta^{-\frac{1}{2\alpha-1}}d\delta \\
& = L \left[ \frac{\delta^{1-\frac{1}{2\alpha-1}}}{1-\frac{1}{2\alpha-1}} \right]_0^1 = \left(\frac{(4\alpha-1)K^2}{2\alpha-1}\right)^{\frac{1}{2\alpha-1}}\frac{2\alpha-1}{2\alpha-2}. \qedhere
\end{align*}
\end{proof}


\begin{remark}
\emph{Note that RUCB uses the upper-confidence bounds (Line 3 of Algorithm \ref{alg:RUCB}) introduced in the original version of UCB \cite{auer2002ucb} (up to the $\alpha$ factor). Recently refined upper-confidence bounds (such as UCB-V \cite{audibert2009ucbv} or KL-UCB \cite{cappe2012klucb}) have improved performance for the regular $K$-armed bandit problem. However, in our setting the arm distributions are Bernoulli and the comparison value is 1/2. Thus, since we have $2\Delta_i^2 \leq kl(p_{1,i}, 1/2) \leq 4\Delta_i^2$ (where $kl(a,b)=a\log\frac{a}{b}+(1-a)\log\frac{1-a}{1-b}$ is the KL divergence between Bernoulli distributions with parameters $a$ and $b$), we deduce that using KL-UCB instead of UCB does not improve the leading constant in the logarithmic term of the regret by a numerical factor of more than 2.}
\end{remark}

\section{Experiments}
\label{sec:experiments}


%
%
%


\begin{figure*}[!t]
\centering
\includegraphics[width=.32\textwidth]{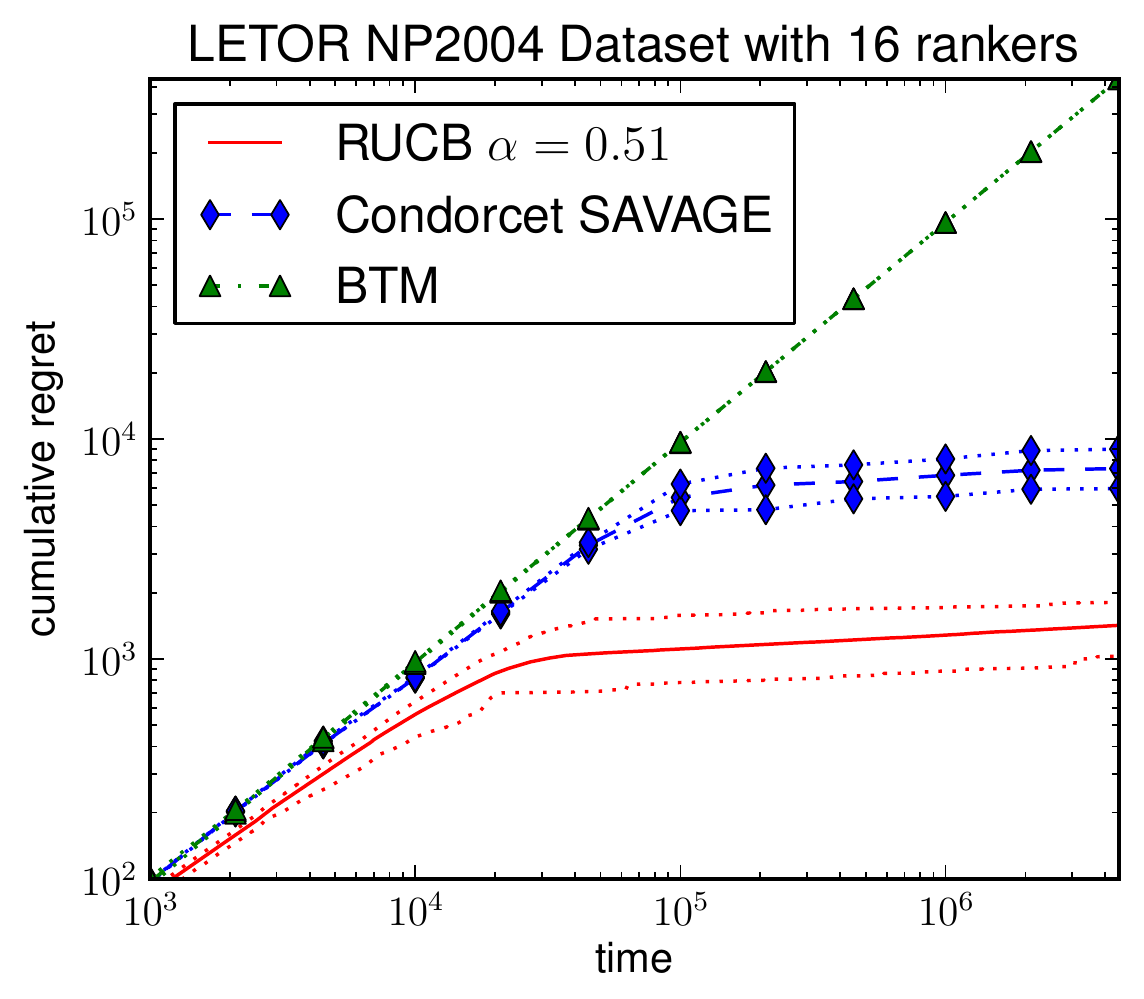}
\includegraphics[width=.32\textwidth]{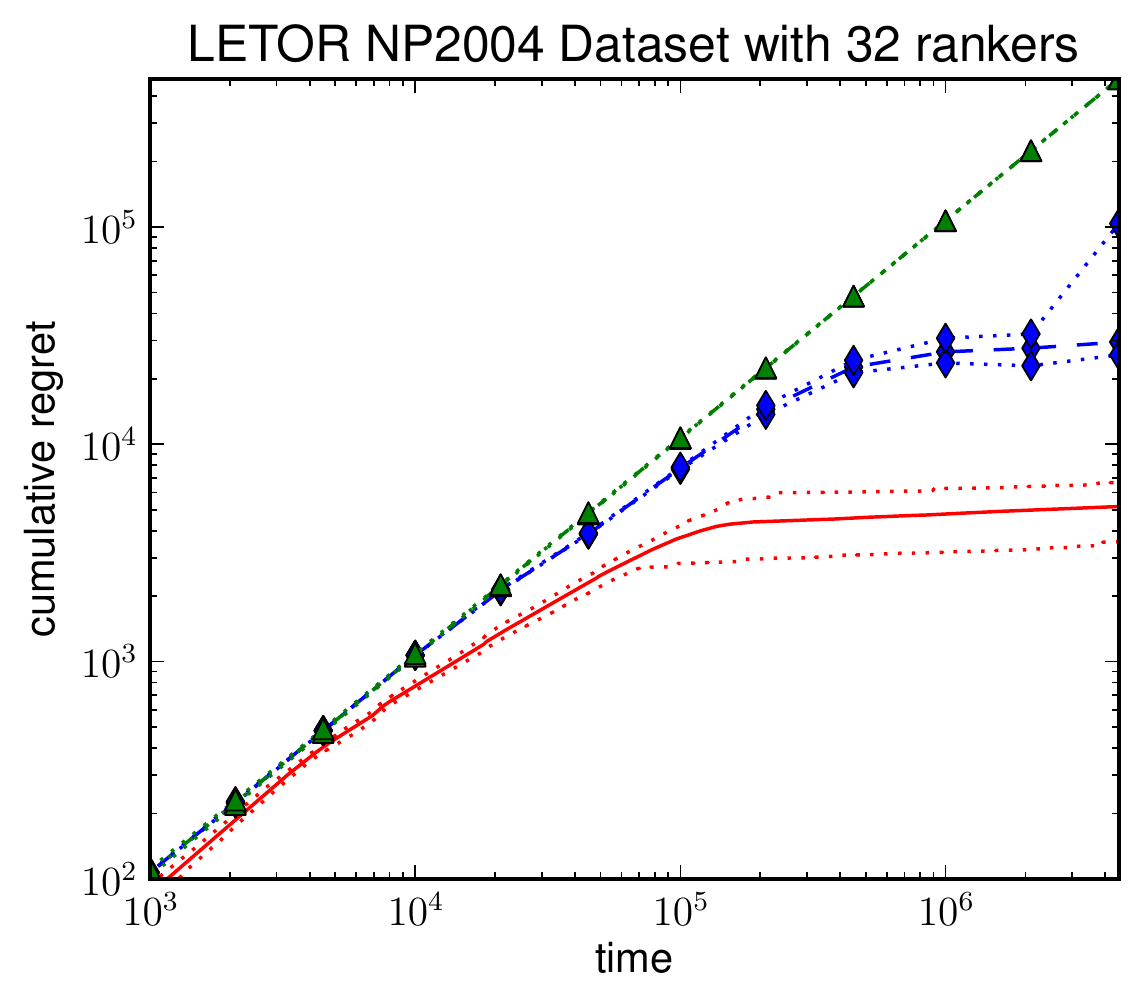}
\includegraphics[width=.32\textwidth]{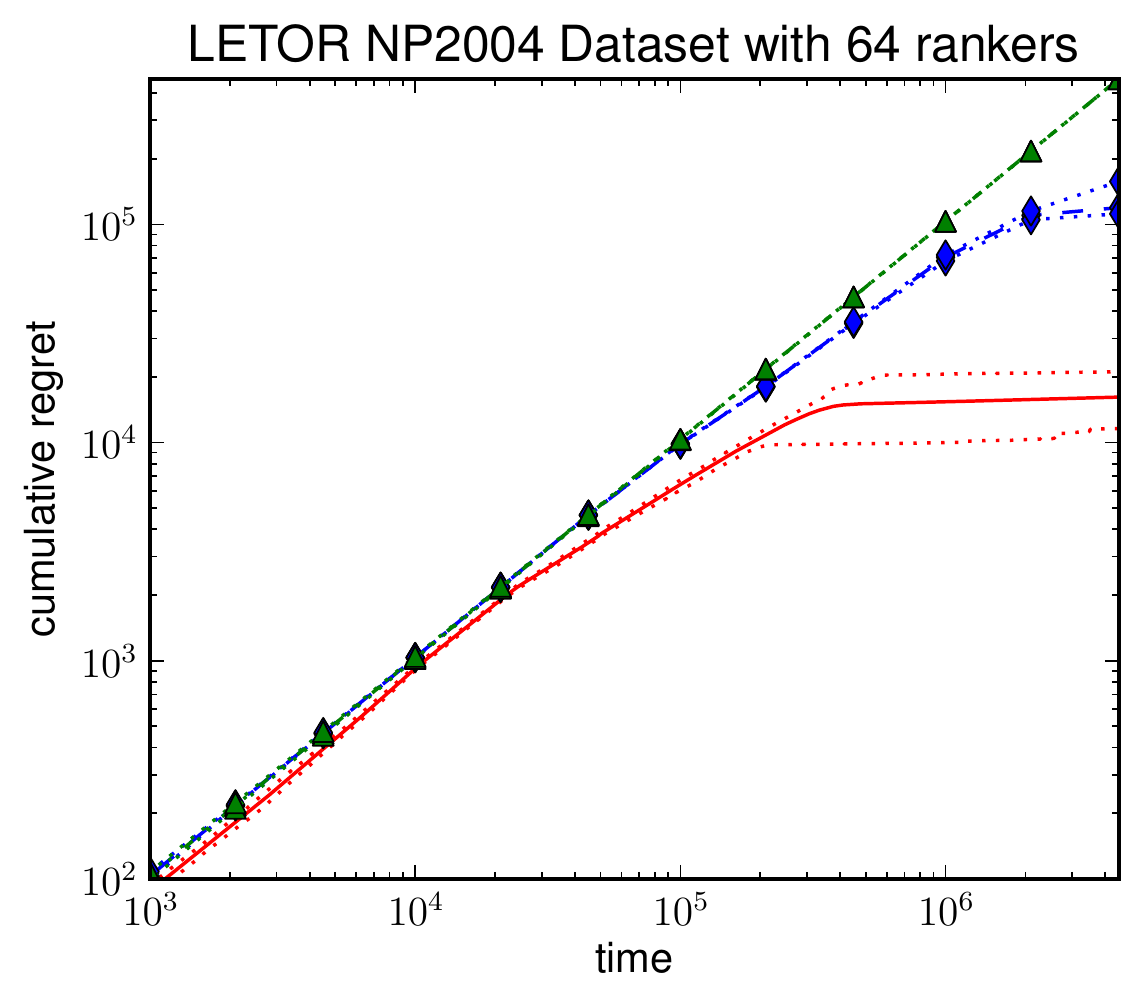}
\includegraphics[width=.32\textwidth]{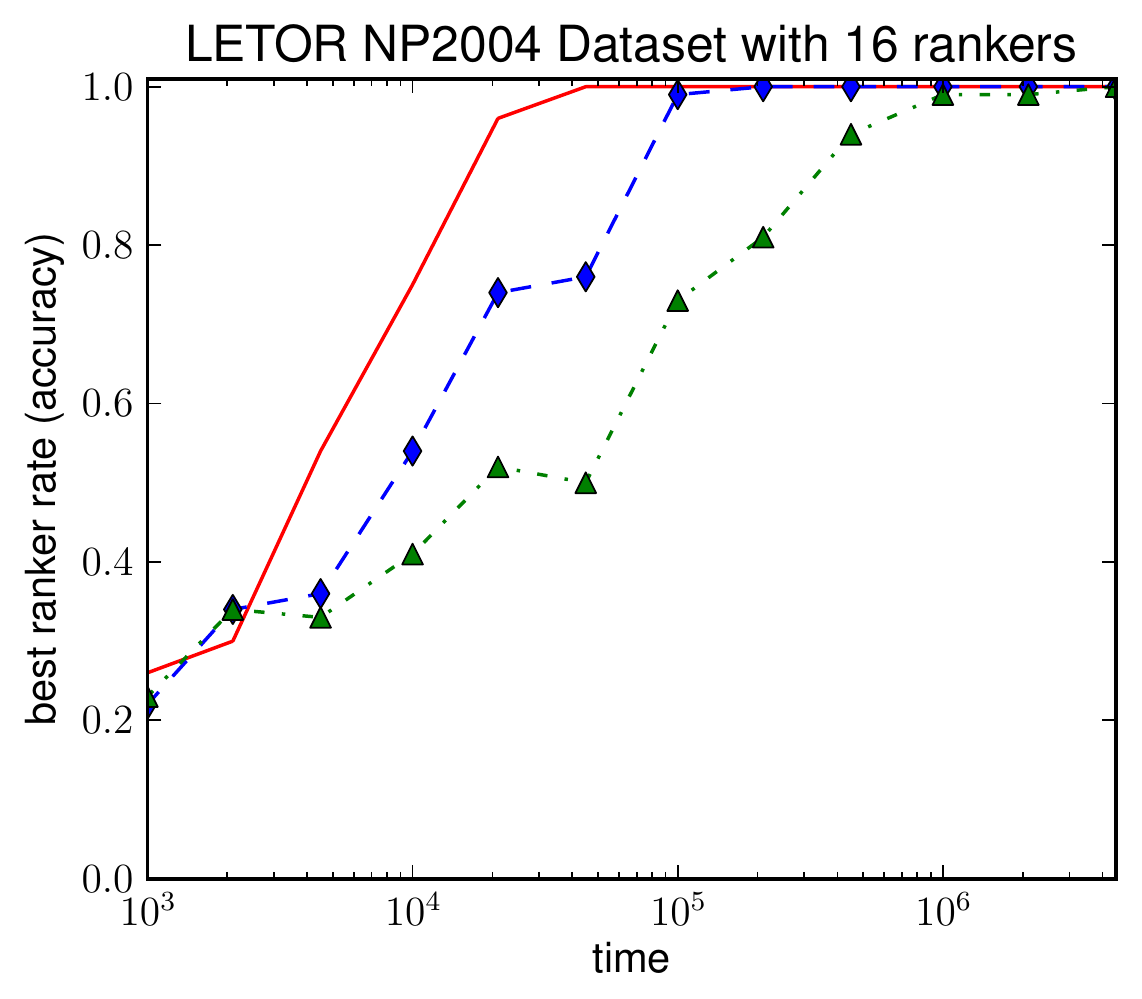}
\includegraphics[width=.32\textwidth]{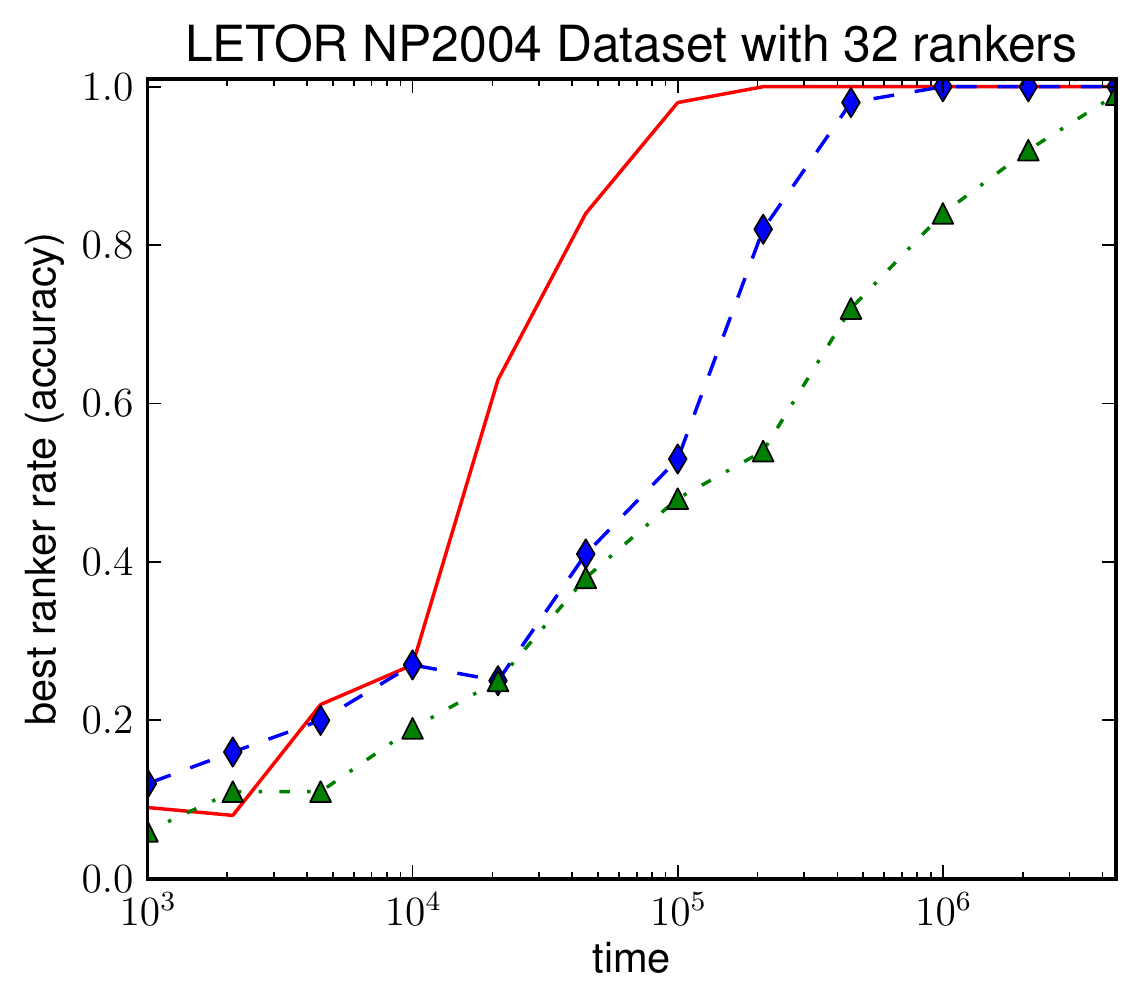}
\includegraphics[width=.32\textwidth]{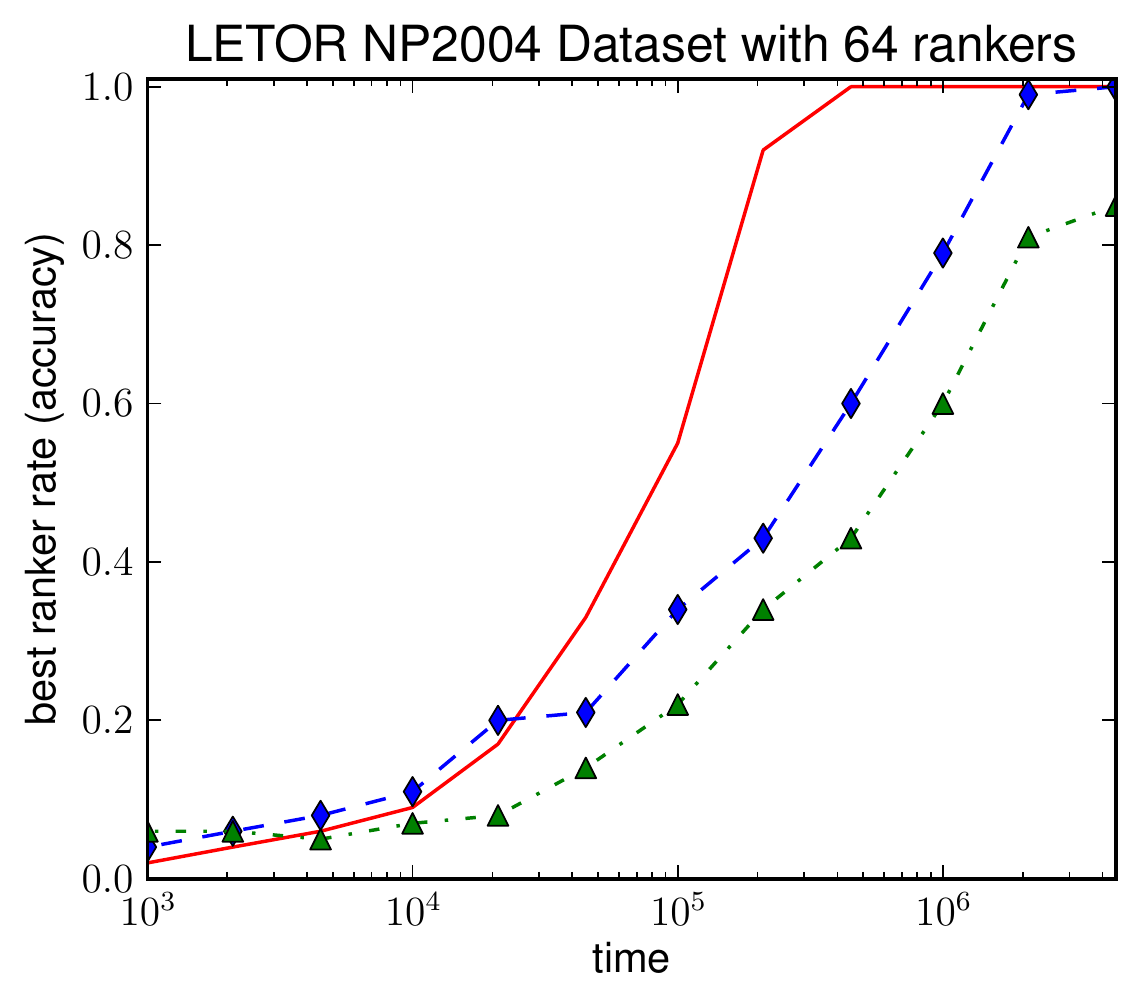}

\vspace{-3mm}

\caption{Average cumulative regret and accuracy for 100 runs of BTM, Condorcet SAVAGE and RUCB with $\alpha=0.51$ applied to three $K$-armed dueling bandit problems with $K=16,32,64$. In the top row of plots, both axes use log scales, and the dotted curves signify best and worst regret performances; in the bottom plots, only the x-axis uses a log scale. 
}
\label{fig:accuracy-regret}

\vspace{-1mm}

\end{figure*}

To evaluate RUCB, we apply it to the problem of \emph{ranker evaluation} from the field of \emph{information retrieval} (IR)~\citep{mann:intr08}.  A ranker is a function that takes as input a user's search query and ranks the documents in a collection according to their relevance to that query.  Ranker evaluation aims to determine which among a set of rankers performs best.  One effective way to achieve this is to use \emph{interleaved comparisons}~\citep{radlinski2008:how}, which interleave the documents proposed by two different rankers and presents the resulting list to the user, whose resulting click feedback is used to infer a noisy preference for one of the rankers.  Given a set of $K$ rankers, the problem of finding the best ranker can then be modeled as a $K$-armed dueling bandit problem, with each arm corresponding to a ranker.


Our experimental setup is built on real IR data, namely the LETOR NP2004 dataset~\citep{letor}. 
Using this data set, we create a set of 64 rankers, each corresponding to a ranking feature provided in the data set, e.g., PageRank.  The ranker evaluation task thus corresponds to determining which single feature constitutes the best ranker~\citep{hofmann:irj13}. 

To compare a pair of rankers, we use \emph{probabilistic interleave} (PI)~\citep{hofmann11:probabilistic}, a recently developed method for interleaved comparisons. To model the user's click behavior on the resulting interleaved lists, we employ a probabilistic user model~\citep{,hofmann11:probabilistic,craswell08:experimental} that uses as input the manual labels (classifying documents as relevant or not for given queries) provided with the LETOR NP2004 dataset. Queries are sampled randomly and clicks are generated probabilistically by conditioning on these assessments in a way that resembles the behavior of an actual user \citep{guo09:tailoring,guo09:efficient}. 

Following \citep{YueJoachims:2011}, we first used the above approach to estimate the comparison probabilities $p_{ij}$ for each pair of rankers and then used these probabilities to simulate comparisons between rankers. More specifically, we estimated the full preference matrix by performing $4000$ interleaved comparisons on each pair of the $64$ feature rankers included in the LETOR dataset.


We evaluated RUCB, Condorcet SAVAGE and BTM using randomly chosen subsets from the pool of $64$ rankers, yielding $K$-armed dueling bandit problems with $K \in \{16,32,64\}$. For each set of rankers, we performed 100 independent runs of each algorithm for a maximum of 4.5 million iterations. For RUCB we set $\alpha=0.51$, which approaches the limit of our high-probability theoretical results, i.e., $\alpha>0.5$ as in Theorem \ref{thm:HighProbBound}. We did not include an evaluation of IF, since both BTM and Condocet SAVAGE were shown to outperform it \cite{YueJoachims:2011,Urvoy:2013}. 

Since BTM and SAVAGE require the exploration horizon as input, we ran $\textup{BTM}_T$ and $\textup{CSAVAGE}_T$ for various horizons $T$ ranging from 1000 to 4.5 million. In the top row of plots in Figure \ref{fig:accuracy-regret}, the markers on the green and the blue curves show the regret accumulated by $\textup{BTM}_T$ and $\textup{CSAVAGE}_T$ in the first $T$ iteration of the algorithm for each of these horizons. Thus, each marker corresponds, not to the continuation of the runs that produced the previous marker, but to new runs conducted with a larger $T$. 

Since RUCB is horizonless, we ran it for 4.5 million iterations and plotted the cumulative regret, as shown using the red curves in the same plots. In the case of all three algorithms, the solid line shows the expected cumulative regret averaged across all 100 runs and the dotted lines show the minimum and the maximum cumulative regret that was observed across runs. Note that these plots are in log-log scale.

The bottom plots in Figure \ref{fig:accuracy-regret} show the accuracy of all three algorithms across 100 runs, computed at the same times as the exploration horizons used for BTM and SAVAGE in the regret plots. Note that these plots are in lin-log scale.

These results clearly demonstrate that RUCB identifies the best arm more quickly, since it asymptotically accumulates 5 to 10 times less regret than Condorcet SAVAGE, while reaching higher levels of accuracy in roughly $20\%$ of the time as Condorcet SAVAGE, all without knowing the horizon $T$. The contrast is even more stark when comparing to BTM.

\section{Conclusions}
\label{sec:conclusions}


This paper proposed a new method called Relative Upper Confidence Bound (RUCB) for the \emph{$K$-armed dueling bandit problem} that extends the Upper Confidence Bound (UCB) algorithm to the relative setting by using optimistic estimates of the pairwise probabilities to choose a potential champion and conducting regular UCB with the champion as the benchmark.  

We proved finite-time high-probability and expected regret bounds of order $\mathcal O(\log t)$ for our algorithm and evaluated it empirically in an information retrieval application. Unlike existing results, our regret bounds hold for all time steps, rather than just a specific horizon $T$ input to the algorithm.  Furthermore, they rely on less restrictive assumptions or have better multiplicative constants than existing methods.  Finally, the empirical results showed that RUCB greatly outperforms state-of-the-art methods.



In future work, we will consider two extensions to this research. First, building off extensions of UCB to the continuous bandit setting \cite{Srinivas:2010,Bubeck:2011,Munos:2011,deFreitas:2012,Valko:2013}, we aim to extend RUCB to the continuous dueling bandit setting, without a convexity assumption as in \cite{Yue:2009}. Second, building off Thompson Sampling \cite{thompson1933likelihood,Agrawal:2012,Kauffman:2012}, an elegant and effective sampling-based alternative to UCB, we will investigate whether a sampling-based extension to RUCB would be amenable to theoretical analysis. Both these extensions involve overcoming not only the technical difficulties present in the regular bandit setting, but also those that arise from the two-stage nature of RUCB. 

\section*{Acknowledgments} 

This research was partially supported by
the European Community's Seventh Framework Programme (FP7/2007-2013) under
grant agreement nr 288024 (LiMoSINe project),
the Netherlands Organisation for Scientific Research (NWO)
under project nrs
640.\-004.\-802, 
727.\-011.\-005, 
612.001.116, 
HOR-11-10, 
the Center for Creation, Content and Technology (CCCT),
the QuaMerdes project funded by the CLARIN-nl program,
the TROVe project funded by the CLARIAH program,
the Dutch national program COMMIT,
the ESF Research Network Program ELIAS,
the Elite Network Shifts project funded by the Royal Dutch Academy of Sciences (KNAW),
the Netherlands eScience Center under project number 027.012.105
and
the Yahoo! Faculty Research and Engagement Program.

\clearpage

\bibliographystyle{icml2014}
\setlength{\bibsep}{7pt}
\bibliography{icml2014-rucb}

\begin{thebibliography}{29}
\providecommand{\natexlab}[1]{#1}
\providecommand{\url}[1]{\texttt{#1}}
\expandafter\ifx\csname urlstyle\endcsname\relax
  \providecommand{\doi}[1]{doi: #1}\else
  \providecommand{\doi}{doi: \begingroup \urlstyle{rm}\Url}\fi

\bibitem[Agrawal \& Goyal(2012)Agrawal and Goyal]{Agrawal:2012}
Agrawal, S. and Goyal, N.
\newblock Analysis of thompson sampling for the multi-armed bandit problem.
\newblock In \emph{Conference on Learning Theory}, pp.\  1--26, 2012.

\bibitem[Audibert et~al.(2009)Audibert, Munos, and
  Szepesv\'{a}ri]{audibert2009ucbv}
Audibert, J.-Y., Munos, R., and Szepesv\'{a}ri, C.
\newblock Exploration-exploitation tradeoff using variance estimates in
  multi-armed bandits.
\newblock \emph{Theor. Comput. Sci.}, 410\penalty0 (19):\penalty0 1876--1902,
  2009.

\bibitem[Auer et~al.(2002)Auer, Cesa-Bianchi, and Fischer]{auer2002ucb}
Auer, P., Cesa-Bianchi, N., and Fischer, P.
\newblock Finite-time analysis of the multiarmed bandit problem.
\newblock \emph{Machine Learning}, 47\penalty0 (2-3):\penalty0 235--256, 2002.

\bibitem[Bubeck et~al.(2009)Bubeck, Munos, and Stoltz]{Bubeck:2009}
Bubeck, S., Munos, R., and Stoltz, G.
\newblock Pure exploration in multi-armed bandits problems.
\newblock In \emph{Algorithmic Learning Theory}, 2009.

\bibitem[Bubeck et~al.(2011)Bubeck, Munos, Stoltz, and Szepesvari]{Bubeck:2011}
Bubeck, S., Munos, R., Stoltz, G., and Szepesvari, C.
\newblock X-armed bandits.
\newblock \emph{Journal of Machine Learning Research}, 12:\penalty0 1655--1695,
  2011.

\bibitem[Capp{\'e} et~al.(2013)Capp{\'e}, Garivier, Maillard, Munos, and
  Stoltz]{cappe2012klucb}
Capp{\'e}, O., Garivier, A., Maillard, O.-A., Munos, R., and Stoltz, G.
\newblock Kullback-{L}eibler upper confidence bounds for optimal sequential
  allocation.
\newblock \emph{Annals of Statistics}, 41\penalty0 (3):\penalty0 1516--1541,
  2013.

\bibitem[Cesa-Bianchi \& Lugosi(2006)Cesa-Bianchi and
  Lugosi]{Cesa-Bianchi:2006}
Cesa-Bianchi, N. and Lugosi, G.
\newblock \emph{Prediction, Learning, and Games}.
\newblock Cambridge University Press, New York, 2006.

\bibitem[Craswell et~al.(2008)Craswell, Zoeter, Taylor, and
  Ramsey]{craswell08:experimental}
Craswell, N., Zoeter, O., Taylor, M., and Ramsey, B.
\newblock An experimental comparison of click position-bias models.
\newblock In \emph{WSDM '08}, pp.\  87--94, 2008.

\bibitem[{de Freitas} et~al.(2012){de Freitas}, Smola, and
  Zoghi]{deFreitas:2012}
{de Freitas}, N., Smola, A., and Zoghi, M.
\newblock Exponential regret bounds for {Gaussian} process bandits with
  deterministic observations.
\newblock In \emph{ICML}, 2012.

\bibitem[F{\"u}rnkranz \& H{\"u}llermeier(2010)F{\"u}rnkranz and
  H{\"u}llermeier]{furnkranz2010}
F{\"u}rnkranz, J. and H{\"u}llermeier, E. (eds.).
\newblock \emph{Preference Learning}.
\newblock Springer-Verlag, 2010.

\bibitem[F{\"u}rnkranz et~al.(2012)F{\"u}rnkranz, H{\"u}llermeier, Cheng, and
  Park]{furnkranz2012towards}
F{\"u}rnkranz, J., H{\"u}llermeier, E., Cheng, W., and Park, S.H.
\newblock Towards preference-based reinforcement learning.
\newblock \emph{Machine Learning}, 2012.

\bibitem[Guo et~al.(2009{\natexlab{a}})Guo, Li, and Faloutsos]{guo09:tailoring}
Guo, F., Li, L., and Faloutsos, C.
\newblock {Tailoring click models to user goals}.
\newblock In \emph{WSCD '09}, pp.\  88--92, 2009{\natexlab{a}}.

\bibitem[Guo et~al.(2009{\natexlab{b}})Guo, Liu, and Wang]{guo09:efficient}
Guo, F., Liu, C., and Wang, Y.
\newblock Efficient multiple-click models in web search.
\newblock In \emph{WSDM '09}, pp.\  124--131, New York, NY, USA,
  2009{\natexlab{b}}. ACM.

\bibitem[Hofmann et~al.(2011)Hofmann, Whiteson, and
  de~Rijke]{hofmann11:probabilistic}
Hofmann, K., Whiteson, S., and de~Rijke, M.
\newblock A probabilistic method for inferring preferences from clicks.
\newblock In \emph{CIKM '11}, pp.\  249--258, USA, 2011. ACM.

\bibitem[Hofmann et~al.(2013)Hofmann, Whiteson, and de~Rijke]{hofmann:irj13}
Hofmann, K., Whiteson, S., and de~Rijke, M.
\newblock Balancing exploration and exploitation in listwise and pairwise
  online learning to rank for information retrieval.
\newblock \emph{Information Retrieval}, 16\penalty0 (1):\penalty0 63--90, 2013.

\bibitem[Joachims(2002)]{joachims2002:optimizing}
Joachims, T.
\newblock Optimizing search engines using clickthrough data.
\newblock In \emph{KDD '02}, pp.\  133--142, 2002.

\bibitem[Kauffmann et~al.(2012)Kauffmann, Korda, and Munos]{Kauffman:2012}
Kauffmann, E., Korda, N., and Munos, R.
\newblock Thompson sampling: an asymptotically optimal finite time analysis.
\newblock In \emph{International Conference on Algorithmic Learning Theory},
  2012.

\bibitem[Lai \& Robbins(1985)Lai and Robbins]{lai85:bandit-lb}
Lai, T.~L. and Robbins, H.
\newblock Asymptotically efficient adaptive allocation rules.
\newblock \emph{Advances in Applied Mathematics}, 6\penalty0 (1):\penalty0 4 --
  22, 1985.

\bibitem[Liu et~al.(2007)Liu, Xu, Qin, Xiong, and Li]{letor}
Liu, T.-Y., Xu, J., Qin, T., Xiong, W., and Li, H.
\newblock Letor: Benchmark dataset for research on learning to rank for
  information retrieval.
\newblock In \emph{LR4IR '07, in conjunction with SIGIR '07}, 2007.

\bibitem[Manning et~al.(2008)Manning, Raghavan, and Sch\"utze]{mann:intr08}
Manning, C., Raghavan, P., and Sch\"utze, H.
\newblock \emph{Introduction to Information Retrieval}.
\newblock Cambridge University Press, 2008.

\bibitem[Munos(2011)]{Munos:2011}
Munos, R.
\newblock Optimistic optimization of a deterministic function without the
  knowledge of its smoothness.
\newblock In \emph{NIPS}, 2011.

\bibitem[Radlinski et~al.(2008)Radlinski, Kurup, and
  Joachims]{radlinski2008:how}
Radlinski, F., Kurup, M., and Joachims, T.
\newblock {How does clickthrough data reflect retrieval quality?}
\newblock In \emph{CIKM '08}, pp.\  43--52, 2008.

\bibitem[Srinivas et~al.(2010)Srinivas, Krause, Kakade, and
  Seeger]{Srinivas:2010}
Srinivas, N., Krause, A., Kakade, S.~M., and Seeger, M.
\newblock Gaussian process optimization in the bandit setting: No regret and
  experimental design.
\newblock In \emph{ICML}, 2010.

\bibitem[Thompson(1933)]{thompson1933likelihood}
Thompson, W.R.
\newblock On the likelihood that one unknown probability exceeds another in
  view of the evidence of two samples.
\newblock \emph{Biometrika}, pp.\  285--294, 1933.

\bibitem[Urvoy et~al.(2013)Urvoy, Clerot, {F\'eraud}, and Naamane]{Urvoy:2013}
Urvoy, T., Clerot, F., {F\'eraud}, R., and Naamane, S.
\newblock Generic exploration and k-armed voting bandits.
\newblock In \emph{ICML}, 2013.

\bibitem[Valko et~al.(2013)Valko, Carpentier, and Munos]{Valko:2013}
Valko, M., Carpentier, A., and Munos, R.
\newblock Stochastic simultaneous optimistic optimization.
\newblock In \emph{ICML}, 2013.

\bibitem[Yue \& Joachims(2009)Yue and Joachims]{Yue:2009}
Yue, Y. and Joachims, T.
\newblock Interactively optimizing information retrieval systems as a dueling
  bandits problem.
\newblock In \emph{ICML}, 2009.

\bibitem[Yue \& Joachims(2011)Yue and Joachims]{YueJoachims:2011}
Yue, Y. and Joachims, T.
\newblock Beat the mean bandit.
\newblock In \emph{ICML}, 2011.

\bibitem[Yue et~al.(2012)Yue, Broder, Kleinberg, and Joachims]{yue12:k-armed}
Yue, Y., Broder, J., Kleinberg, R., and Joachims, T.
\newblock {The K-armed dueling bandits problem}.
\newblock \emph{Journal of Computer and System Sciences}, 78\penalty0
  (5):\penalty0 1538--1556, September 2012.

\end{thebibliography}

\clearpage

\section{Appendix}

\begin{figure}[!b]
\centering
\includegraphics[width=.485\textwidth]{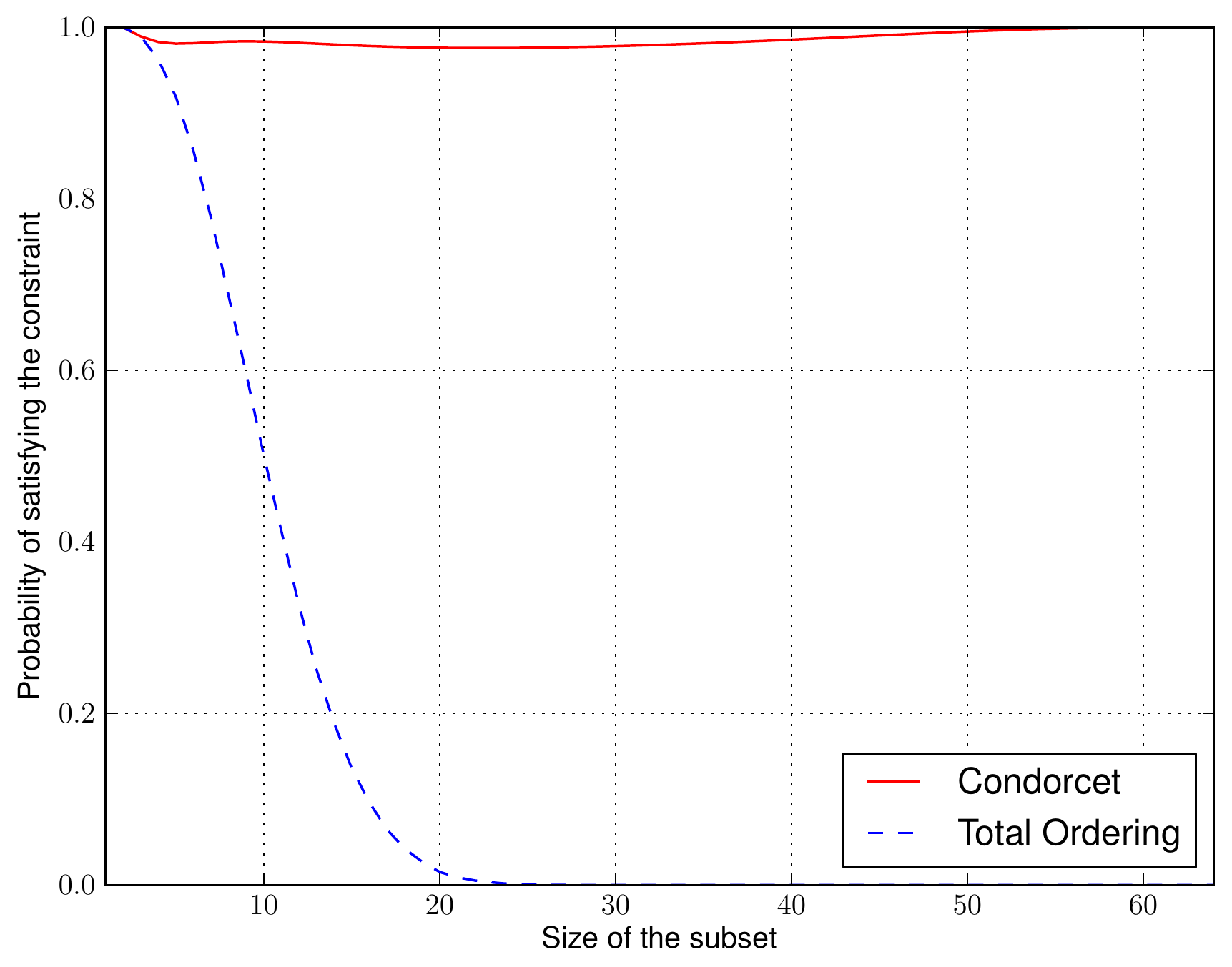}
\caption{The probability that the Condorcet and the total ordering assumptions hold for subsets of the feature rankers. The probability is shown as a function of the size of the subset.}
\label{fig:prob-condorcet}
\end{figure}

\begin{figure*}[!t]
\centering
\includegraphics[width=\textwidth]{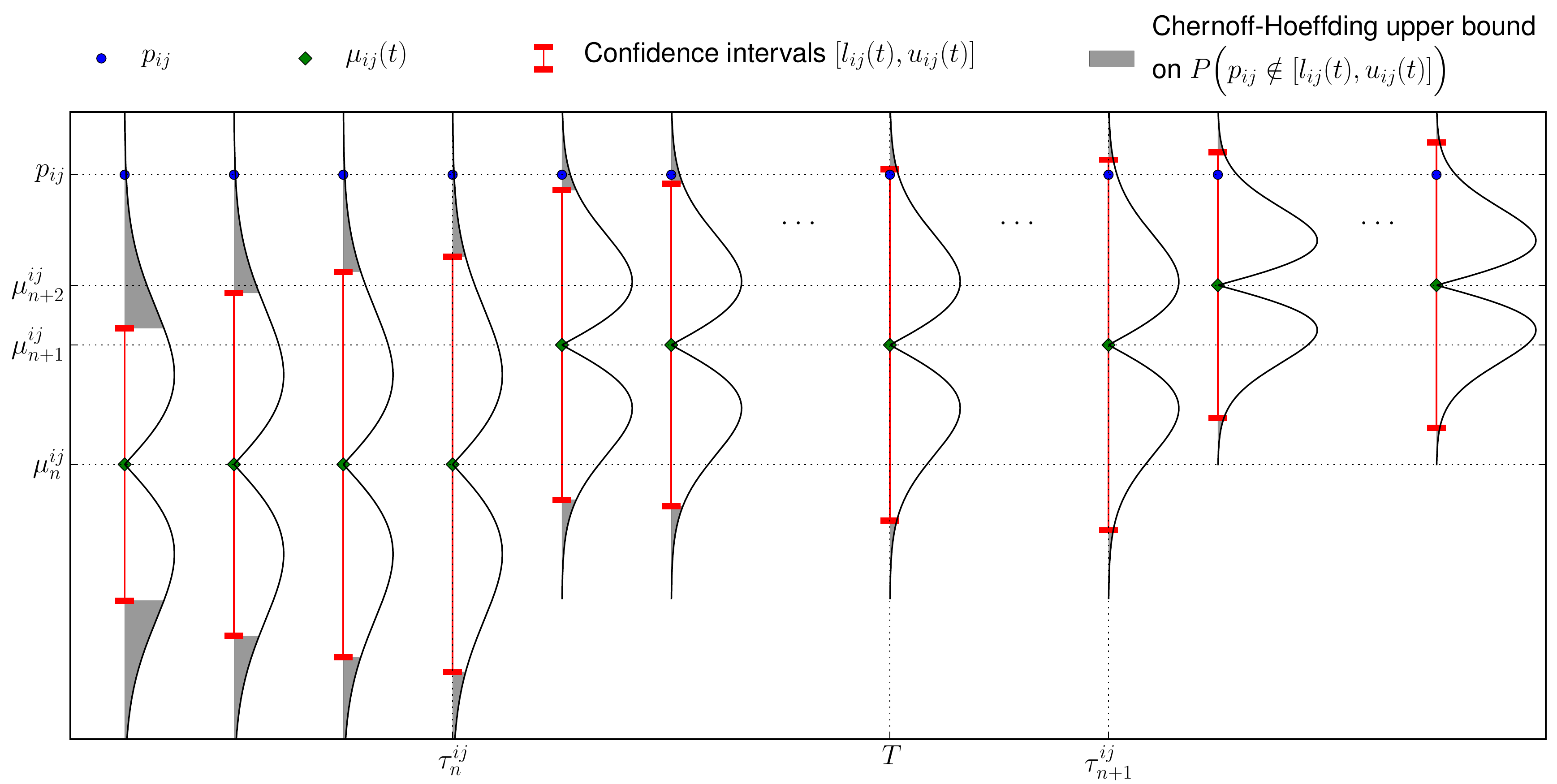}


\caption{An illustrations of the idea behind Lemma \ref{lem:HighProbBound} using an example of how the confidence intervals of a single pair of arms $(a_i,a_j)$, and their relation to the comparison probability $p_{ij}$, might evolve over time.  The time-step $\tau^{ij}_m$ denotes the $m^{\textup{th}}$ time when the arms $a_i$ and $a_j$ were chosen by RUCB to be compared against each other. We also define $\mu^{ij}_m := \mu_{ij}(\tau^{ij}_m)$. The time $T$ is when the confidence intervals $[l_{ij}(t),u_{ij}(t)]$ begin to include $p_{ij}$. The lemma then states that with probability $1-\delta$, we have $T \leq C(\delta)$. \\ \vspace{-2mm} \\ Moreover, for each time-step, the area of the shaded region under the vertical graphs is the bound given by the Chernoff-Hoeffding (CH) bound on the probability that the confidence interval will not contain $p_{ij}$. Note that the CH bound has the form $e^{-(x-\mu^{ij}_n)^2}$ and so in order for this number to be the area under a graph (hence making it easier to illustrate in a figure), we have drawn the derivative of this function, $f^{ij}_n(x) := |x-\mu^{ij}_n| e^{-(x-\mu^{ij}_n)^2}$, which is why the graphs are equal to $0$ in the middle. Note that this does not mean that $\mu^{ij}_n$ has very low probability of being close to $p_{ij}$: the graphs drawn here are not the PDFs of the posteriors, but simply a manifestation of the bound given by the Chernoff-Hoeffding bound. More specifically, the property that they satisfy is that $P\Big(p_{ij} \notin [l_{ij}(t),u_{ij}(t)]\Big) \leq \int_{-\infty}^{l_{ij}(t)} f^{ij}_{N_{ij}(t)}(x)dx + \int_{u_{ij}(t)}^{\infty} f^{ij}_{N_{ij}(t)}(x)dx$.}
\label{fig:lemma1Appendix}
\end{figure*}



Here we provide some details that were alluded to in the main body of the paper.

\subsection{The Condorcet Assumption}

As mentioned in Section \ref{sec:relatedwork}, IF and BTM require the comparison probabilities $p_{ij}$ to satisfy certain difficult to verify conditions. Specifically, IF and BTM require a \emph{total ordering} $\{a_1,\ldots,a_K\}$ of the arms to exist such that $p_{ij} > \frac{1}{2}$ for all $i < j$. Here we provide evidence that this assumption is often violated in practice. By contrast, the algorithm we propose in Section \ref{sec:algorithm} makes only the Condorcet assumption, which is implied by the total ordering assumption of IF and BTM.





In order to test how stringent an assumption the existence of a Condorcet winner is compared the total ordering assumption, we estimated the probability of each assumption holding in our ranker evaluation application. 
Using the same preference matrix as in our experiments in Section \ref{sec:experiments}, we computed for each $K=1,\ldots,64$ the probability $P_K$ that a given $K$-armed dueling bandit problem obtained from considering $K$ of our $64$ feature rankers would have a Condorcet winner as follows: first, we calculated the number of $K$-armed dueling bandit that have a Condorcet winner by calculating for each feature ranker $r$ how many $K$-armed duelings bandits it can be the Condorcet winner of: for each $r$, this is equal to $\binom{N_r}{K}$, where $N_r$ is the number rankers that $r$ beats; next, we divided this total number of $K$-armed dueling bandit with a Condorcet winner by $\binom{64}{K}$, which is the number of all $K$-armed dueling bandit that one could construct from these $64$ rankers. 

The probabilities $P_K$, plotted as a function of $K$ in Figure \ref{fig:prob-condorcet} (the red curve), were all larger than $0.97$. The same plot also shows an estimate of the probability that the total ordering assumption holds for a given $K$ (the blue curve), which was obtained by randomly selecting $100,000$ $K$-armed bandits and searching for ones that satisfy the total ordering assumption. As can be seen from Figure \ref{fig:prob-condorcet}, as $K$ grows the probability that the total ordering assumption holds decreases rapidly. This is because there exist cyclical relationships between these feature rankers and as soon as the chosen subset of feature rankers contains one of these cycles, it fails to satisfy the total ordering condition. By contrast, the Condorcet assumption will still be satisfied as long as the cycle does not include the Condorcet winner. Moreover, because of the presence of these cycles, the probability that the Condorcet assumption holds decreases initially as $K$ increases, but then increases again because the number of all possible $K$-armed dueling bandit decreases as $K$ approaches $64$.

Furthermore, in addition to the total ordering assumption, IF and BTM each require a form of \emph{stochastic transitivity}.  In particular, IF requires \emph{strong stochastic transitivity}; for any triple $(i,j,k)$, with $i < j < k$, the following condition needs to be satisfied:
\[ p_{ik} \geq \max\{p_{ij},p_{jk}\}.  \]
BTM requies the less restrictive \emph{relaxed stochastic transitivity}, i.e., that there exists a number $\gamma \geq 1$ such that for all pairs $(j,k)$ with $1 < j < k$, we have
\[ \gamma p_{1k} \geq \max\{p_{1j},p_{jk}\}. \]
As pointed out in \cite{YueJoachims:2011}, strong stochastic transitivity is often violated in practice, a phenomenon also observed in our experiments: for instance, all of the $K$-armed dueling bandit on which we experimented require $\gamma > 1$.

Even though BTM permits a broader class of $K$-armed dueling bandit problems, it requires $\gamma$ to be explicitly passed to it as a parameter, which poses substantial difficulties in practice. If $\gamma$ is underestimated, the algorithm can in certain circumstances be misled with high probability into choosing the Borda winner instead of the Condorcet winner, e.g., when the Borda winner has a larger average advantage over the remaining arms than the Condorcet winner. On the other hand, though overestimating $\gamma$ does not cause the algorithm to choose the wrong arm, it nonetheless results in a severe penalty, since it makes the algorithm much more exploratory, yielding the $\gamma^7$ term in the upper bound on the cumulative regret, as discussed in Section \ref{sec:relatedwork}.

\subsection{Proof of Lemma \ref{lem:HighProbBound}}

In this section, we prove Lemma \ref{lem:HighProbBound}, whose statement is repeated here for convenience. Recall from Section \ref{sec:theory} that we assume without loss of generality that $a_1$ is the optimal arm. Moreover, given any $K$-armed dueling bandit algorithm, we define $w_{ij}(t)$ to be the number of times arm $a_i$ has beaten $a_j$ in the first $t$ iterations of the algorithm. We also define $u_{ij}(t) := \frac{w_{ij}(t)}{w_{ij}(t)+w_{ji}(t)} + \sqrt{\frac{\alpha\ln t}{w_{ij}(t)+w_{ji}(t)}}$, where $\alpha$ is any positive contant, and $l_{ij}(t) := 1-u_{ji}(t)$. Moreover, for any $\delta > 0$, define $C(\delta) := \left(\frac{(4\alpha-1)K^2}{(2\alpha-1)\delta}\right)^{\frac{1}{2\alpha-1}}$.

{\bf Lemma \ref{lem:HighProbBound}.}~\emph{Let $\vP := \left[ p_{ij} \right]$ be the preference matrix of a $K$-armed dueling bandit problem with arms $\{a_1,\ldots,a_K\}$, satisfying $p_{1j} > \frac{1}{2}$ for all $j > 1$ (i.e., $a_1$ is the Condorcet winner). Then, for any dueling bandit algorithm and any $\alpha > \frac{1}{2}$ and $\delta > 0$, we have}
\begin{equation}\label{eqn:TailBound}  P\Big( \forall\,t>C(\delta),i,j,\; p_{ij} \in [l_{ij}(t),u_{ij}(t)] \Big) > 1-\delta. \end{equation}

\begin{proof}
To decompose the lefthand side of \eqref{eqn:TailBound}, we introduce the notation $\mathcal{G}_{ij}(t)$ for the ``good'' event that at time $t$ we have $p_{ij} \in [l_{ij}(t),u_{ij}(t)]$, which satisfies the following:
\begin{itemize}[leftmargin=*,topsep=0pt,parsep=0pt,partopsep=0pt]
\item[] \hspace{-5mm} (i) $\mathcal{G}_{ij}(t) = \mathcal{G}_{ji}(t)$ because of the triple of equalities $\Big(p_{ji},l_{ji}(t),u_{ji}(t)\Big) = \Big(1-p_{ij},1-u_{ij}(t),1-l_{ij}(t)\Big)$.
\item[] \hspace{-6mm} (ii) $\mathcal{G}_{ii}(t)$ always holds, since $\left(p_{ii},l_{ii}(t),u_{ii}(t)\right) = \left(\frac{1}{2},\frac{1}{2},\frac{1}{2}\right)$. Together with (i), this means that we only need to consider $\mathcal{G}_{ij}(t)$ for $i < j$.
\item[] \hspace{-4.5mm} (iii) Define $\tau^{ij}_n$ to be the iteration at which arms $i$ and $j$ were compared against each other for the $n^{th}$ time. If $G_{ij}\left(\tau^{ij}_n+1\right)$ holds, then the events $\mathcal{G}_{ij}(t)$ hold for all $t \in \left(\tau^{ij}_n,\tau^{ij}_{n+1}\right]$ because when $t \in \left(\tau^{ij}_n,\tau^{ij}_{n+1}\right]$, $w_{ij}$ and $w_{ji}$ remain constant and so in the expressions for $u_{ij}(t)$ and $u_{ji}(t)$ only the $\ln t$ changes, which is a monotonically increasing function of $t$. So, we have 
\[ l_{ij}(t) \leq l_{ij}(\tau^{ij}_n+1) \leq p_{ij} \leq u_{ij}(\tau^{ij}_n+1) \leq u_{ij}(t). \]

Moreover, the same statement holds with $\tau^{ij}_n$ replaced by any $T \in \left(\tau^{ij}_n,\tau^{ij}_{n+1}\right]$, i.e., if we know that $\mathcal{G}_{ij}(T)$ holds, then $\mathcal{G}_{ij}(t)$ also holds for all $t \in \left(T,\tau^{ij}_{n+1}\right]$. This is illustrated in Figure \ref{fig:lemma1Appendix}. 
\end{itemize}

Now, given the above three facts, we have for any $T$
\begin{align}\label{eqn:goodPruning}
& P\Big( \forall\,t\geq T,i,j,\; \mathcal{G}_{ij}(t) \Big) \\
& \; = P\Big( \forall\,i>j,\; \mathcal{G}_{ij}(T) \textup{ and } \forall\,n\;s.t.\; \tau^{ij}_n > T, \; \mathcal{G}_{ij}(\tau^{ij}_n) \Big). \nonumber
\end{align}
%
%
Let us now flip things around and look at the complement of these events, i.e. the ``bad'' event $\mathcal{B}_{ij}(t)$ that $p_{ij} \notin [l_{ij}(t),u_{ij}(t)]$ occurs. 
Then, subtracting both sides of Equation \eqref{eqn:goodPruning} from $1$ and using the union bound gives
\begin{align*}
& P\Big( \exists\,t>T,i,j\;s.t.\; \mathcal{B}_{ij}(t) \Big) \\
& \leq \sum_{i<j} \bigg[ P\Big( \mathcal{B}_{ij}(T) \Big) + P\Big( \exists\,n: \tau^{ij}_n > T \textup{ and } \mathcal{B}_{ij}(\tau^{ij}_n) \Big) \bigg]. 
\end{align*}
Further decomposing the righthand side using union bounds and making the condition explicit, we get

\vspace{-7mm}

\begin{align*}
& P\Big( \exists\,t>T,i,j\;s.t.\; \mathcal{B}_{ij}(t) \Big) \\
& \leq \sum_{i>j} \Bigg[ P\left( \left|p_{ij}-\mu^{ij}_{N_{ij}(T)}\right| > \sqrt{\frac{\alpha\ln T}{N_{ij}(T)}} \right) + \\
& P\left( \exists\,n \leq T\;s.t.\; \tau^{ij}_n > T \textup{ and } \left|p_{ij}-\mu^{ij}_n\right| > \sqrt{\frac{\alpha\ln \tau^{ij}_n}{n}} \right) \\
& \qquad\quad + P\left( \exists\,n > T\;s.t.\; \left|p_{ij}-\mu^{ij}_n\right| > \sqrt{\frac{\alpha\ln \tau^{ij}_n}{n}} \right) \Bigg], \\
\end{align*}

\vspace{-10mm}

since $T < n < \tau^{ij}_n$. Here, $\mu^{ij}_n := \frac{w_{ij}(\tau^{ij}_n)}{w_{ij}(\tau^{ij}_n)+w_{ji}(\tau^{ij}_n)}$ is the frequentist estimate of $p_{ij}$ after $n$ comparisons between arms $a_i$ and $a_j$. 

Now, in the above sum, we can upper-bound the first term by looking at the higher probability event that $\mathcal{B}_{ij}(T)$ happens for any possible number of comparisons between $a_i$ and $a_j$, and since we know that $N_{ij}(T) \leq T$, we can replace $N_{ij}(T)$ with a variable $n$ that can take values between $0$ and $T$. For the second term, we know that $\tau^{ij}_n > T$, so we can replace $\tau^{ij}_n$ with $T$ and remove the condition $\tau^{ij}_n > T$ and look at all $n \leq T$. For the third term, since we always have that $n < \tau^{ij}_n$, we can replace $\tau^{ij}_n$ with $n$ and get a higher probability event. Putting all of this together we get the looser bound

\vspace{-5mm}

\begin{align}
	& P\Big( \exists\,t>T,i,j\;s.t.\; \mathcal{B}_{ij}(t) \Big) \nonumber \\
	& \leq \sum_{i<j} \Bigg[ P\left( \exists\, n\in\{0,\ldots,T\}: \left| p_{ij} - \mu^{ij}_n \right| > \sqrt{\frac{\alpha\ln T}{n}} \right) \nonumber \\
	& \qquad\; + P\left( \exists\, n\in\{0,\ldots,T\}: \left| p_{ij} - \mu^{ij}_n \right| > \sqrt{\frac{\alpha\ln T}{n}} \right) \nonumber \\
	& \qquad\; + P\left( \exists\, n > T\; s.t. \; \left| p_{ij} - \mu^{ij}_n \right| > \sqrt{\frac{\alpha\ln n}{n}} \right) \Bigg] \nonumber \\
	& \leq \sum_{i<j} \Bigg[ 2\sum_{n=0}^T P\left( \left| p_{ij} - \mu^{ij}_n \right| > \sqrt{\frac{\alpha\ln T}{n}} \right) \nonumber \\
	& \qquad\quad + \sum_{n=T+1}^\infty P\left( \left| p_{ij} - \mu^{ij}_n \right| > \sqrt{\frac{\alpha\ln n}{n}} \right) \Bigg]. \label{eqn:CHstep}
\end{align}

\vspace{-2mm}

To bound the expression on line \eqref{eqn:CHstep}, we apply the Chernoff-Hoeffding bound, which in its simplest form states that given i.i.d.\ random variables $X_1, \ldots, X_n$, whose support is contained in $[0,1]$ and whose expectation satisfies $\mathbb{E}[X_k] = p$, and defining $\mu_n := \frac{X_1+\cdots+X_n}{n}$, we have $P(|\mu_n-p| > a) \leq 2e^{-2na^2}$. This gives us

\vspace{-5mm}

\begin{align}
	& P\Big( \exists\,t>T,i,j\;s.t.\; \mathcal{B}_{ij}(t) \Big) \nonumber \\
	& \leq \sum_{i<j} \left[ 2\sum_{n=1}^T 2e^{-2\cancel{n}\dfrac{\alpha\ln T}{\cancel{n}}} + \sum_{n=T+1}^\infty 2e^{-2\cancel{n}\dfrac{\alpha\ln n}{\cancel{n}}} \right] \nonumber \\ 
	& = \frac{K(K-1)}{2} \left[ \sum_{n=1}^T \frac{4}{T^{2\alpha}} + \sum_{n=T+1}^\infty \frac{2}{n^{2\alpha}} \right] \nonumber \\
	& \leq \frac{2K^2}{T^{2\alpha-1}} + K^2\int_{T}^\infty \frac{dx}{x^{2\alpha}}, \; \textup{since $\frac{1}{x^{2\alpha}}$ is decreasing.} \nonumber \\
	& \leq \frac{2K^2}{T^{2\alpha-1}} + K^2\int_{T}^\infty \frac{dx}{x^{2\alpha}} = \frac{2K^2}{T^{2\alpha-1}} + \frac{K^2}{(1-2\alpha) x^{2\alpha-1}} \bigg|_{T}^\infty \nonumber \\
	& = \frac{(4\alpha-1)K^2}{(2\alpha-1) T^{2\alpha-1}}. \label{eqn:badUpperBound}
\end{align}

\vspace{-2mm}

Now, since $C(\delta) = \left(\frac{(4\alpha-1)K^2}{(2\alpha-1)\delta}\right)^{\frac{1}{2\alpha-1}}$ for each $\delta~>~0$, the bound in \eqref{eqn:badUpperBound} gives us \eqref{eqn:TailBound}.

\end{proof}

\end{document}